\documentclass[english,a4paper,12pt]{article}

\usepackage{a4wide}

\usepackage{amsxtra, amsfonts, amssymb, amstext, amsmath, mathtools}
\usepackage{amsthm}
\usepackage{booktabs}
\usepackage{nicefrac}
\usepackage{xspace}
\usepackage{url}
\usepackage{graphics,color}
\usepackage[algo2e,ruled,vlined,linesnumbered]{algorithm2e}
\usepackage{wrapfig}

\newtheorem{theorem}{Theorem}
\newtheorem{lemma}[theorem]{Lemma}
\newtheorem{corollary}[theorem]{Corollary}

\newcommand{\oea}{\mbox{$(1 + 1)$~EA}\xspace}

\newcommand{\oplea}{\mbox{$(1+\lambda)$~EA}\xspace}
\newcommand{\mpoea}{\mbox{$(\mu+1)$~EA}\xspace}
\newcommand{\mplea}{\mbox{$(\mu+\lambda)$~EA}\xspace}
\newcommand{\mclea}{\mbox{${(\mu,\lambda)}$~EA}\xspace}
\newcommand{\oclea}{\mbox{$(1,\lambda)$~EA}\xspace}
\newcommand{\opllga}{\mbox{$(1+(\lambda,\lambda))$~GA}\xspace}

\newcommand{\OM}{\textsc{OM}\xspace}
\newcommand{\onemax}{\textsc{OneMax}\xspace}
\newcommand{\LO}{\textsc{Leading\-Ones}\xspace}
\newcommand{\leadingones}{\LO}
\newcommand{\needle}{\textsc{Needle}\xspace}
\newcommand{\cliff}{\textsc{Cliff}\xspace}

\newcommand{\plateau}{\textsc{Plateau}\xspace}
\newcommand{\jump}{\textsc{Jump}\xspace}

\newcommand{\gmax}{g_{\max}}

\newcommand{\X}{\mathcal{X}}
\newcommand{\Geins}{{\textbf{(G1)}}\xspace}
\newcommand{\Gzwei}{{\textbf{(G2)}}\xspace}
\newcommand{\Gdrei}{{\textbf{(G3)}}\xspace}

\newcommand{\R}{\ensuremath{\mathbb{R}}}

\newcommand{\N}{\ensuremath{\mathbb{N}}} 
\newcommand{\Z}{\ensuremath{\mathbb{Z}}}

\newcommand{\calE}{\ensuremath{\mathcal{E}}}

\DeclareMathOperator{\Bin}{Bin}

\newcommand{\eps}{\varepsilon}

\let\originalleft\left
\let\originalright\right
\renewcommand{\left}{\mathopen{}\mathclose\bgroup\originalleft}
\renewcommand{\right}{\aftergroup\egroup\originalright}

\usepackage{hyperref}

\begin{document}
\title{Does Comma Selection Help To Cope With Local Optima?\thanks{This is the full version of a paper~\cite{Doerr20gecco} that appeared at GECCO 2020. It contains all proofs and additional information that had to be omitted from the conference version for reasons of space. }}

\author{Benjamin Doerr\\ Laboratoire d'Informatique (LIX)\\ CNRS\\ \'Ecole Polytechnique\\ Institut Polytechnique de Paris\\ Palaiseau\\ France}

\maketitle

\sloppy{
\begin{abstract}
  One hope when using non-elitism in evolutionary computation is that the ability to abandon the current-best solution aids leaving local optima. To improve our understanding of this mechanism, we perform a rigorous runtime analysis of a basic non-elitist evolutionary algorithm (EA), the $(\mu,\lambda)$ EA, on the most basic benchmark function with a local optimum, the jump function. We prove that for all reasonable values of the parameters and the problem, the expected runtime of the $(\mu,\lambda)$~EA is, apart from lower order terms, at least as large as the expected runtime of its elitist counterpart, the $(\mu+\lambda)$~EA (for which we conduct the first runtime analysis on jump functions to allow this comparison). Consequently, the ability of the $(\mu,\lambda)$~EA to leave local optima to inferior solutions does not lead to a runtime advantage.
  
  We complement this lower bound with an upper bound that, for broad ranges of the parameters, is identical to our lower bound apart from lower order terms. This is the first runtime result for a non-elitist algorithm on a multi-modal problem that is tight apart from lower order terms.
\end{abstract}

\section{Introduction}

The mathematical runtime analysis of evolutionary algorithms (EAs) and other randomized search heuristics is a young but established subfield of the general research area of heuristic search~\cite{NeumannW10,AugerD11,Jansen13,DoerrN20}. This field, naturally, has started with regarding the performance of simple algorithms on simple test problems: The problems usually were \emph{unimodal}, that is, without local optima different from the global optimum, the algorithms were \emph{elitist} and often had \emph{trivial populations}, and the runtime guarantees only estimated the \emph{asymptotic order of magnitude}, that is, gave $O(\cdot)$ upper bounds or $\Omega(\cdot)$ lower bounds. 

Despite this restricted scope, many fundamental results have been obtained and our understanding of the working principles of EAs has significantly increased with these works.
In this work, we go a step further with a \emph{tight} (apart from lower order terms) analysis of how a \emph{non-elitist} evolutionary algorithm with \emph{both non-trivial parent and offspring populations} optimizes a \emph{multimodal} problem. 

In contrast to the practical use of evolutionary algorithms, where non-elitism is often employed, the mathematical analysis of evolutionary algorithms so far could find only little evidence for the use of non-elitism. The few existing works, very roughly speaking (see Section~\ref{ssec:nonel} for more details) indicate that when the selection pressure is large, then the non-elitist algorithm simulates an elitist one, and when the selection pressure is low, then no function with unique global optimum can be optimized efficiently. The gap between these two regimes is typically very small. Consequently, a possible profit from non-elitism would require a careful parameter choice.

One often named advantage of non-elitist algorithms is their ability to leave local optima to inferior solutions, which can reduce the time spent uselessly in local optima. To obtain a rigorous view on this possible advantage, we analyze the performance of the well-known \mclea (see Section~\ref{ssec:algo} for a precise definition) on the multimodal jump function benchmark~\cite{DrosteJW02}. We note that, apart from some sporadic results on custom-tailored example problems and the corollaries from very general results (see Theorems~\ref{thm:oldjumpLB} and~\ref{thm:oldjump} further below), this is the first in-depth study of how a non-elitist algorithm optimizes a classic benchmark with local optima. 

Our main result (see Section~\ref{sec:lower}) is that in this setting, the small middle range between a too small and a too high selection pressure, which could be envisaged from previous works, does not exist. Rather, the two undesired regimes overlap significantly. We note that for the \mclea, the selection pressure is reasonably well described by the ratio of the offspring population size $\lambda$ to the parent population size $\mu$. The selection pressure is low if $\lambda \le (1-\eps) e \mu$ for some constant $\eps > 0$. In this case, the \mclea needs an exponential time to optimize any function $f : \{0,1\}^n \to \R$ with at most a polynomial number of global optima~\cite{Lehre10}. The selection pressure is high if $\lambda \ge (1+\eps) e \mu$ for some constant $\eps > 0$ and if $\lambda$ is at least logarithmic in the envisaged runtime. In this case, the \mclea can optimize many classic benchmark functions in a runtime at most a constant factor slower than, say, the \mplea, see~\cite{Lehre11}. 

Our main result implies (Corollary~\ref{cor:lower}) that already when $\lambda \ge 2 \mu$, $\lambda$ is super-constant, and $\lambda = o(n^{k-1})$, the runtime of the $\mclea$ on all jump functions with jump size $k \le  n^{1-\eps}$ is at least the runtime of the \mplea (apart from lower order terms); to prove this statement, we also conduct the first so far and sufficiently tight runtime analysis of the \mplea on jump functions (Theorem~\ref{thm:plus}). Consequently, the two regimes of a too low selection pressure and of no advantage over the elitist algorithm overlap in the range  $\lambda \in [2\mu, (1-\eps)e\mu]$, leaving for jump functions no space for a middle regime with runtime advantages from non-elitism. We note that our result, while natural, is not obvious. In particular, as a comparison of the $(1,1)$~EA and the \oea on highly deceptive functions shows (see Section~\ref{sssec:nonelhelpful} for more details), it is not always true that the elitist algorithm is at least as good as its non-elitist counterpart. 

Our result does not generally disrecommend to use non-elitism, in particular, it does not say anything about possible other advantages from using non-elitism. Our result, however, does indicate that the ability to leave local optima to inferior solutions is non-trivial to turn into a runtime advantage (whereas at the same time, as observed in previous works, there is a significant risk that the selection pressure is too low to admit any reasonable progress).

We also prove an upper bound for the runtime of the \mclea on jump functions (Theorem~\ref{thm:upper}), which shows that our lower bound for large ranges of the parameters (but, of course, only for $\lambda \ge (1+\eps) e \mu$) is tight including the leading constant. This appears to be the first precise\footnote{We use the term \emph{precise} to denote runtime estimates that are asymptotically tight including the leading constant, that is, where the estimated runtime $\tilde T(n)$ and the true runtime $T(n)$ for problem size $n$ satisfy $\lim\limits_{n \to \infty} \tilde T(n) / T(n) = 1$.} on  runtime result for a non-trivial non-elitist algorithm on a non-trivial problem. 

From the technical perspective, it is noteworthy that we obtain precise bounds in settings where the previously used methods (negative drift for lower bounds, level-based analyses of non-elitist population processes for upper bounds) could not give precise analyses, and in the case of negative drift could usually not even determine the right asymptotic order of the runtime. We are optimistic that our methods will be profitable for other runtime analyses as well.

\section{State of the Art and Our Results}\label{sec:state}

This work progresses the state of the art in three directions with active research in the recent past, namely non-elitist evolutionary algorithms, precise runtime analyses, and methods to prove lower bounds in the presence of negative drift and upper bounds for non-elitist population processes. We now describe these previous states of the art and detail what is the particular progress made in this work. We concentrate ourselves on classic evolutionary algorithms (also called genetic algorithms) for the optimization in discrete search spaces. We note that non-elitism has been used in other randomized search heuristics such as the Metropolis algorithm~\cite{JansenW07,WangZD21}, simulated annealing~\cite{Jansen05,Wegener05}, strong-selection-weak-mutation (SSWM)~\cite{PaixaoHST17,OlivetoPHST18}, and memetic algorithms~\cite{NguyenS20}. Letting the selection decisions not only depend on the fitness, e.g., in tabu search or when using fitness sharing, also introduces some form of non-elitism. From a broader perspective, also many probabilistic model building algorithms such as ant colony optimizers or estimation-of-distribution algorithms can be seen as non-elitist, since they often allow moves to inferior models. From an even broader point of view, even restart strategies can be seen as a form of non-elitism. While all these research directions are interesting, it seems to us that the results obtained there, to the extent that we understand them, are not too closely related to our results and therefore not really comparable.

\subsection{Non-Elitist Algorithms}\label{ssec:nonel}

While non-elitist evolutionary algorithms are used a lot in practice, the mathematical theory of EAs so far was not very successful in providing convincing evidences for the usefulness of non-elitism. This might be due to the fact that rigorous research on non-elitist algorithms has started only relatively late, caused among others by the fact that many non-elitist algorithms require non-trivial populations, which form another challenge for mathematical analyses. Another reason, naturally, could be that non-elitism is not as profitable as generally thought. Our work rather points into the latter direction.

The previous works on non-elitist algorithms can roughly be grouped as follows.

\subsubsection{Exponential Runtimes When the Selection Pressure is Low}\label{sssec:nonellow}

 By definition, non-elitist algorithms may lose good solutions. When this happens too frequently (low selection pressure), then the EA finds it hard to converge to good solutions, resulting in a poor performance

The first to make this empirical observation mathematically precise in a very general manner was Lehre in his remarkable work~\cite{Lehre10}. For a broad class of non-elitist population-based EAs, he gives conditions on the parameters that imply that the EA cannot optimize any pseudo-Boolean function $f : \{0,1\}^n \to \R$ with at most a polynomial number of optima in time sub-exponential in $n$. Due to their general nature, we have to restrict ourselves here to what Lehre's results imply for the \mclea, but we note that analogous results hold for a much wider class of algorithms. For the \mclea using the usual mutation rate~$\frac 1n$, Lehre shows that when $\lambda \le (1-\eps) e \mu$, where $\eps > 0$ is any positive constant, then the time to find a (global) optimum of any pseudo-Boolean function with at most a polynomial number of optima is exponential in $n$ with high probability. 

We note that more specific results showing the danger of a too low selection pressure have appeared earlier. For example, already in 2007 J\"agersk\"upper and Storch~\cite[Theorem~1]{JagerskupperS07} showed that the \oclea with $\lambda \le \frac{1}{14} \ln(n)$ is inefficient on any pseudo-Boolean function with a unique optimum. The range of $\lambda$ for which such a negative performance is observed was later extended to the asymptotically tight value $\lambda \le (1-\eps) \log_{\frac{e}{e-1}} n$ by Rowe and Sudholt~\cite{RoweS14}. Happ, Johannsen, Klein, and Neumann~\cite{HappJKN08} showed that two simple (1+1)-type hillclimbers using fitness proportionate selection in the choice of the surviving individual are not efficient on any linear function with positive weights. Neumann, Oliveto, and Witt~\cite{NeumannOW09} showed that a mutation-only variant of the Simple Genetic Algorithm with fitness proportionate selection is inefficient on the \onemax function when the population size $\mu$ is at most polynomial, and it is inefficient on any pseudo-Boolean function with unique global optimum when $\mu \le \frac 14 \ln(n)$. Oliveto and Witt~\cite{OlivetoW15} showed that the true Simple Genetic Algorithm (using crossover) cannot optimize \onemax efficiently when $\mu \le n^{\frac 14 - \eps}$.

We note that the methods in~\cite{Lehre10} were also used to prove lower bounds for particular objective functions. The following result was given for a variant of jump functions~\cite[Theorem~5]{Lehre10}. To be precise, a similar result was proven for a tournament-selection algorithm and it was stated that an analogous statement, which we believe to be the following, holds for the \mclea as well. As a reviewer pointed out, it is not immediately clear how the proof in~\cite{Lehre10} extends to the classic definition of jump functions.

\begin{theorem}[cf.~Lehre~\cite{Lehre10}]\label{thm:oldjumpLB}
  Let $n \in \Z_{\ge 1}$, $\eps > 0$ a constant, $k \le (0.2 - \eps)n$, and $k = \omega(\log n)$. Let $f_k : \{0,1\}^n \to \R$ be defined by $f_k(x) = \onemax(x)$ when $\onemax(x) \notin [n-k+1..n-1]$ and $f_k(x) = 0$ otherwise. Then the expected runtime of the \mclea with polynomial $\lambda$ on $\jump_{nk}$ is at least $\exp(\Omega(k))$, where all asymptotics is for $n$ tending to infinity.
\end{theorem}

\subsubsection{Pseudo-Elitism When the Selection Pressure is High}\label{sssec:nonelhigh}

When a non-elitist algorithm has the property that, despite the theoretical danger of losing good solutions, it very rarely does so, then its optimization behavior becomes very similar to the one of an elitist algorithm. Again the first to make this effect precise for a broad class of algorithms was Lehre in his first paper on level-based arguments for non-elitist populations~\cite{Lehre11}.  

Lehre's fitness-level theorem for non-elitist population-based algorithms assumes that the search space can be partitioned into levels such that (i)~the algorithm has a reasonable chance to sample a solution in some level $j$ or higher once a constant fraction of the population is at least on level $j-1$ (``base level'') and (ii)~there is an exponential growth of the number of individuals on levels higher than the base level; more precisely (but still simplified), if there are $\mu_0 < \gamma_0 \mu$ individuals above the base level, then in the next generation the number of individuals above the base level follows a binomial distribution with parameters $\mu$ and $p = (1+\delta) \frac{\mu_0}{\mu}$, where $\gamma_0$ and $\delta$ are suitable parameters. If in addition the population sizes involved are large enough, then (again very roughly speaking) the runtime of the algorithm is at most a constant factor larger than the runtime guarantee which could be obtained for an elitist analogue of the non-elitist EA. From the assumptions made here, it cannot be excluded that the non-elitist EA loses a best-so-far solution; however, due to the exponential growth of condition~(ii) and the sufficiently large population size, this can only happen if there are few individuals above the base level. Hence the assumptions of the level-based method, roughly speaking, impose that that the EA behaves like an elitist algorithm except when it just has found a new best solution. In this case, with positive probability (usually at most some constant less than one) the new solution is lost. This constant probability for losing a new best individual (and the resulting need to re-generate one) may lead to a constant-factor loss in the runtime, but not more. Very roughly speaking, one can say that such non-elitist EAs, while formally non-elitist algorithms, do nothing else than a slightly slowed-down emulation of an elitist algorithm. That said, it has to be remarked that both proving level-based theorems (see \cite{Lehre11,DangL16algo,CorusDEL18,DoerrK21}) and applying them (see also~\cite{DangLN19}) is technical and much less trivial than what the rule of thumb ``high selection pressure imitates elitism'' suggests. 

For the optimization of jump functions via the \mclea, the work~\cite{CorusDEL18} implies the following result. We note that it was shown only to the variant of jump functions  regarded in Theorem~\ref{thm:oldjumpLB} (where the fitness in the gap region is uniformly zero), but from the proofs it is clear that the result also holds for the standard definition~\cite{DrosteJW02} used in this work.

\begin{theorem}[\cite{CorusDEL18}]\label{thm:oldjump}
  Let $k \in [1..n]$. Let $\eps > 0$ be a constant and let $c$ be a sufficiently large constant (depending on $\eps$). Let $\lambda \ge c k \ln(n)$ and $\mu \le \frac{\lambda}{(1+\eps)e}$. Then runtime $T$ of the \mclea on $\jump_{nk}$ satisfies 
  \[E[T] = O(n^k + n \lambda + \lambda \log \lambda).\]
\end{theorem}
  
For the particular case of the \oclea and the \oplea, J\"agersk\"upper and Storch in an earlier work also gave fitness-level theorems~\cite[Lemma~6 and~7]{JagerskupperS07}. They also showed that both algorithms essentially behave identical for $t$ iterations when $\lambda$ is at least logarithmic in~$t$~\cite[Theorem~4]{JagerskupperS07}. This effect, without quantifying $\lambda$ and without proof, was already proposed in~\cite[p.~415]{JansenJW05}. J\"agersk\"upper and Storch show in particular that when $\lambda \ge 3 \ln n$, then the \oclea optimizes \onemax in asymptotically the same time as the \oplea. The actual runtimes given for the \oclea in~\cite{JagerskupperS07} are not tight since at that time a tight analysis of the \oplea on \onemax was still missing; however, it is clear that the arguments given in~\cite{JagerskupperS07} also allow to transfer the tight upper bound of $O(n \log n + \lambda n \frac{\log\log n}{\log n})$ for the \oplea from~\cite{DoerrK15} to the \oclea. 

The minimum value of $\lambda$ that ensures an efficient optimization of the \oclea on \onemax was lowered to the asymptotically tight value of $\lambda \ge \log_{\frac{e}{e-1}} n \approx 2.18 \ln n$ in~\cite{RoweS14}. Again, only the upper bound of $O(n \log n + n \lambda)$ was shown. We would not be surprised if with similar arguments also a bound of $O(n \log n + \lambda n \frac{\log\log n}{\log n})$ could be shown, but this is less obvious here than for the result of~\cite{JagerskupperS07}. 

For the benchmark function \leadingones, the threshold between a superpolynomial runtime of the \oclea and a runtime asymptotically equal to the one of the \oplea was shown to be at $\lambda = (1 \pm \eps) 2 \log_{\frac{e}{e-1}} n$~\cite{RoweS14}.

\subsubsection{Examples Where Non-Elitism is Helpful}\label{sssec:nonelhelpful}

The dichotomy described in the previous two subsections suggests that it is not easy to find examples where non-elitism is useful. This is indeed true apart from two exceptions. 

J\"agersk\"upper and Storch~\cite{JagerskupperS07} constructed an artificial example function that is easier to optimize for the \oclea than for the \oplea. The \cliff function $\cliff: \{0,1\}^n \to \N$ is defined by $\cliff(x) = \OM(x)$ if $\OM(x) < n - \lfloor n/3 \rfloor$ and $\cliff(x) = \OM(x) - \lfloor n/3 \rfloor$ otherwise. J\"agersk\"upper and Storch showed that the \oclea with $\lambda \ge 5 \ln n$ optimizes $\cliff$ in an expected number of $O(\exp(5 \lambda))$ fitness evaluations, whereas the \oplea with high probability needs at least $n^{n/4}$ fitness evaluations. While this runtime difference is enormous, it has to be noted that even for the best value of $\lambda = 5 \ln n$, the runtime guarantee for the \oclea is only $O(n^{25})$. Also, we remark that the local optimum of the $\cliff$ function has a particular structure which helps to leave the local optimum: Each point on the local optimum has $\lfloor n/3 \rfloor$ neighbors from which it is easy to hill-climb to the global optimum (as long as one does not use a steepest ascent strategy). Also, for each point on the local optimum there are $\Omega(n^2)$ search points in Hamming distance two from which any local search within less than $n/3$ improvements finds the global optimum.
This is a notable difference to the $\jump_{nk}$ function, where hill-climbing from any point of the search space that is not the global optimum or one of its $n$ neighbors surely leads to the local optimum. We would suspect that such larger radii of attraction are closer to the structure of difficult real-world problems, but we leave it to the reader to decide which model is most relevant for their applications.
%

We note that a second, albeit extreme and rather academic, example for an advantage of non-elitism is implicit in the early work~\cite{GarnierKS99} by Garnier, Kallel, and Schoenauer. They show that the $(1,1)$ EA on any function $f : \{0,1\}^n \to \R$ with unique global optimum has an expected optimization time of $(1+o(1)) \frac{e}{e-1} 2^n$; this follows from Proposition~3.1 in their work. 
When taking a highly deceptive function like the trap function, this runtime is significantly better than the ones of elitist algorithms, which typically are $n^{\Theta(n)}$. Of course, all this is not overly surprising -- the $(1,1)$ EA uses no form of selection and hence just performs a random walk in the search space (where the one-step distribution is given by the mutation operator). Therefore, this algorithm does not suffer from the deceptiveness of the trap function as do elitist algorithms. Also, a runtime reduction from $n^{\Theta(n)}$ to $\exp(\Theta(n))$ clearly is not breathtaking. Nevertheless, this is a second example where a \mclea significantly outperforms the corresponding \mplea.

Since in this work we are only interested in how non-elitism (and more specifically, comma selection) helps to leave local optima \emph{and} by this improve runtimes, we do not discuss in detail other motivations for employing non-elitist algorithms. We note brief{}ly, though, that comma selection is usually employed in self-adaptive algorithms. Self-adaptation means that some algorithm parameters are stored as part of the genome of the individuals and are subject to variation together with the original individual. The hope is that this constitutes a generic way to adjust algorithm parameters. When using plus selection together with self-adaptation, there would be the risk that the population at some point only contains individuals with unsuitable parameter values. Now variation will only generate inferior offspring. These will not be accepted and, consequently, the parameter values encoded in the genome of the individuals cannot be changed. When using comma selection, it is possible to accept individuals with inferior fitness, and these may have superior parameter values. We are not aware of a rigorous demonstration of this effect, but we note that the two runtime analysis papers~\cite{DangL16ppsn,DoerrWY21} on self-adaptation both use comma selection. We further note that comma selection is very common in continuous optimization, in particular, in evolution strategies, but since it is generally difficult to use insights from continuous optimization in discrete optimization and vice-versa we do not discuss results from continuous optimization here. 

\subsubsection{Our Contribution}\label{sssec:nonelour}

In Section~\ref{sec:lower}, we show that for all interesting values of the parameters of the problem and the algorithm, the expected runtime of the \mclea on jump functions is, apart from possibly lower order terms, at least the expected runtime of the \mplea. This shows that for this problem, there can be no significant advantage of using comma selection. 

Our upper bound in Theorem~\ref{thm:upper}, provided mostly to show that our analysis is tight including the leading constant, improves Theorem~\ref{thm:oldjump} by making the leading constant precise and being applicable for all offspring population sizes $\lambda \ge C \ln(n)$, $C$ a constant independent of the jump size~$k$. To the best of our knowledge, this is the first time that the runtime of a non-elitist algorithm was proven with this precision.

\subsection{Precise Runtime Analyses}

Traditionally, algorithm analysis aims at gaining a rough understanding how the runtime of an algorithm depends on the problem size. As such, most results only show statements on the asymptotic order of magnitude of the runtime, that is, results in big-Oh notation. For classic algorithmics, this is justified among others by the fact that the predominant performance measure, the number of elementary operations, already ignores constant factor differences in the execution times of the elementary operations. 

In evolutionary computation, where the classic performance measure is the number of fitness evaluations, this excuse for ignoring constant factors is not valid, and in fact, in the last few years more and more \emph{precise runtime results} have appeared, that is, results which determine the runtime asymptotically precise apart from lower order terms. Such results are useful, obviously because constant factors matter in practice, but also because many effects are visible only at constant-factor scales. For example, it was shown in~\cite{DoerrG13algo} that all $\Theta(\frac 1n)$ mutation rates lead to a $\Theta(n \log n)$ runtime of the \oea on all pseudo-Boolean linear functions, but only Witt's seminal result~\cite{Witt13} that the runtime is $(1+o(1)) \frac{e^c}{c} n \ln n$ for the mutation rate~$\frac cn$, $c>0$ a constant, allows to derive that $\frac 1n$ is the asymptotically best mutation rate.  

Overall, not too many non-trivial precise runtime results are known. In a very early work~\cite{GarnierKS99}, it was shown that the \oea with mutation rate $\frac cn$ optimizes the \onemax function in an expected time of $(1+o(1)) \frac{e^c}{c} n \ln n$ and the $\needle$ function in time $(1+o(1)) \frac{1}{1-e^c} 2^n$. More than ten years later, in independent works~\cite{BottcherDN10,Sudholt13} the precise runtime of the \oea on \leadingones was determined; here~\cite{BottcherDN10} also regarded general mutation rates and deduced from their result that the optimal mutation rate of approximately $\frac{1.59}n$ is higher than the usual recommendation~$\frac 1n$, and that a fitness dependent mutation rate gives again slightly better results (this was also the first time that a fitness dependent parameter choice was proven to be superior to static rates by at least a constant factor difference in the runtime). Precise runtime results for a broader class of $(1+1)$-type algorithms on \leadingones have recently appeared in~\cite{Doerr19tcs}. A series of recent works~\cite{AlanaziL14,LissovoiOW17,DoerrLOW18} obtained precise runtimes of different hyper-heuristics on \leadingones and thus allowed to discriminate them by their runtime. The precise expected runtime of the \oea with general unbiased mutation operator on the $\plateau_k$ function was determined~\cite{AntipovD18} to be $(1+o(1)) \binom{n}{k} p_{1:k}^{-1}$, where $p_{1:k}$ is the probability that the mutation operator flips between one and $k$ bits. Apparently, here the details of the mutation operator are not very important -- only the probability to flip between one and $k$ bits has an influence on the runtime.

The only precise runtime analysis for an algorithm with a non-trivial population can be found in~\cite{GiessenW17}, where the runtime of the \oplea with mutation rate $\frac cn$, $c$ a constant, on \onemax was shown to be $(1+o(1))(\frac{e^c}{c} n \ln n + n \lambda \frac{\ln\ln\lambda}{2 \ln \lambda})$. This result has the surprising implication that here the mutation rate is only important when $\lambda$ is small.

The only precise runtime analysis for a multi-modal objective function was conducted in~\cite{DoerrLMN17}, where the runtime of the \oea with arbitrary mutation rate was determined for jump functions; this work led to the development of a heavy-tailed mutation operator that appears to be very successful~\cite{FriedrichQW18,FriedrichGQW18,FriedrichGQW18heavysubm,WuQT18,AntipovBD20gecco}.

In summary, there is only a small number of precise runtime analyses, but many of them could obtain insights that would not have been possible with less precise analyses. 

\emph{Our result}, an analysis of the \mclea on jump functions that is precise for $k \le 0.1n$, $\lambda = o(n^{k-1})$, $\lambda \ge (1+\eps) e \mu$, and $\lambda = \Omega(\log n)$ sufficiently large, is the second precise analysis for a population-based algorithm (after~\cite{GiessenW17}),  is the second precise analysis for a multimodal fitness function (after~\cite{DoerrLMN17}), and is the first precise analysis for a non-elitist algorithm (apart from fact that the result~\cite{GiessenW17} could be transfered to the $\oclea$ for large $\lambda$ via the argument~\cite{JagerskupperS07} that in this case the \oplea and the \oclea have essentially identical performances). 

\subsection{Methods: Negative Drift and Level-based Analyses}\label{sec:methods}

To obtain our results, we also develop new techniques for two classic topics, namely the analysis of processes showing a drift away from the target (``negative drift'') and the analysis of non-elitist population processes via level-based arguments.

\subsubsection{Negative Drift}

It is natural that a stochastic process $X_0, X_1, \ldots$ finds it hard to reach a certain target when the typical behavior is taking the process away from the target. \emph{Negative drift theorems} are an established tool for the analysis of such situations. They roughly speaking state the following. Assume that the process starts at some point $b$ or higher, that is, $X_0 \ge b$, and that we aim at reaching a target $a < b$. Assume that whenever the process is above the target value, that is, $X_t > a$, we have an expected progress $E[X_{t+1} - X_t] \ge \delta$, $\delta$ some constant, away from the target, and that this progress satisfies some concentration assumption like two-sided exponential tails. Then the expected time to reach or undershoot the target is at least exponential in the distance $b-a$. 

The first such result in the context of evolutionary algorithms was shown by Oliveto and Witt~\cite{OlivetoW11} (note the corrigendum~\cite{OlivetoW12}). Improved versions were subsequently given in~\cite{RoweS14,OlivetoW15,Kotzing16,LenglerS18,Witt19}. The comprehensive survey~\cite[Section~2.4.3]{Lengler20bookchapter} gives a complete coverage of this topic. What is important to note for our purposes is that (i)~all existing negative drift results are quite technical to use due to the concentration assumptions, that (ii)~they all give a lower bound that is only exponential in the length of the interval in which the (constant) negative drift is observed, and that (iii)~they all (apart from the technical work of Hajek~\cite{Hajek82}) do not give tight bounds, but only bounds of type $\exp(\Omega(b-a))$ with the implicit constant in the exponent not specified.

Earlier than the general negative drift theorem, Lehre~\cite{Lehre10} proved a negative drift theorem for population-based processes via multi-type branching processes. Just as the general negative drift theorems described above, it only gives lower bounds exponential in the length of the negative drift regime and the base of the exponential function is not made explicit. Consequently, in~\cite[Theorem~5]{Lehre10} (Theorem~\ref{thm:oldjumpLB} in this work), only an $\exp(\Omega(k))$ lower bound for the runtime of the \mclea on $\jump_{nk}$ was derived

Since we aim at an $\Omega(n^k)$ lower bound caused by a negative drift in the short gap region (of length $k$) of the jump function, and since further we aim at results that give the precise leading constant of the runtime, we cannot use these tools. We therefore resort to the \emph{additive drift applied to a rescaled process} argument first made explicit in~\cite{AntipovDY19}. The basic idea is very simple: For a suitable function $g : \R \to \R$ one regards the process $(g(X_t))_t$ instead of the original process $(X_t)_t$, shows that it makes at most a slow progress towards the target, say $E[g(X_{t+1}) - g(X_t) \mid X_t > a] \ge -\delta$, and concludes from the classic additive drift theorem~\cite{HeY01} (Theorem~\ref{tadddrift} in this work) that the expected time to reach or undershoot $a$ when starting at $b$ is at least $\frac{g(b) - g(a)}{\delta}$. While the basic approach is simple and natural, the non-trivial part is finding a rescaling function $g$ which both gives at most a slow progress towards the target and gives a large difference $g(b) - g(a)$. The rescalings used in~\cite{AntipovDY19} and~\cite{Doerr19foga} were both of exponential type, that is, $g$ was roughly speaking an exponential function. By construction, they only led to lower bounds exponential in $b-a$, and in both cases the lower bound was not tight (apart from being exponential).

\emph{Our progress:} Hence the technical novelty of this work is that we devise a rescaling for our problem that (i)~leads to a lower bound of order $n^k$ for a process having negative drift only is an interval of length $k$, and (ii) such that these lower bounds are tight including the leading constant. Clearly, our rescalings (as all rescalings used previously) are specific to our problem. Nevertheless, they demonstrate that the rescaling method, different from the classic negative drift theorems, can give very tight lower bounds and lower bounds that are super-exponential in the length of the interval in which the negative drift is observed. We are optimistic that such rescalings will find other applications in the future.

Note added in proof: For reasons of completeness of this discussion on lower bound methods, we note that between the submission of this work in May~2020 and the first notification in June 2021, a further lower bound method called \emph{negative multiplicative drift} was proposed~\cite{Doerr20ppsnLB}. Different from what a reviewer suggests, it is not in any way related to the rescaling method. While we do not want to rule out that it can also be employed to prove our lower bound, it is clear that this would either also need a rescaling of the process (and then our approach appears more direct) or it would need estimates on the change of the maximum \onemax-value in the population that are substantially different from ours in the proof of Theorem~\ref{thm:lower}.

\subsubsection{Level-based Analyses}

While level-based arguments for the analysis of non-elitist algorithms have been used much earlier, see, e.g.,~\cite{Eremeev99}, the fitness-level analysis of Lehre~\cite{Lehre11} might still be the first general method to analyze non-elitist population-based processes. We gave a high-level description of this method in Section~\ref{sssec:nonelhigh} and we will give a more detailed discussion in Section~\ref{ssec:level} to enable us to prove our upper bound. For this reason, we now explain without further explanations what is our progress over the state of the art of this method.

Similar to the state of the art in negative drift theorems, all existing variants of the level-based methods do not give results that are tight including the leading constant. Also, from the complexity of the proofs of these results, it appears unlikely that such tight results can be obtained in the near future. 

\emph{Our progress:} For our problem of optimizing jump functions, we can exploit the fact that the most difficult, and thus time consuming, step is generating the global optimum from a population that has fully converged into the local optimum. To do so, we use the non-tight level-based methods only up to the point when the population only consists of local optima (we call this an \emph{almost perfect population}). This can be done via a variation of the existing level-based results (Corollary~\ref{cor:level}). From that point on, we estimate the remaining runtime by computing the waiting time for generating the optimum from a local optimum. Of course, since we are analyzing a non-elitist process, we are not guaranteed to keep an almost perfect population. For that reason, we also need to analyze the probability of losing an almost perfect population and to set up a restart argument to regain an almost perfect population. Naturally, this has to be done in a way that the total runtime spent here is only a lower-order fraction of the time needed to generate the global optimum from an almost perfect population. 

A side effect of this approach is that we only need a logarithmic offspring population size, that is, it suffices to have $\lambda \ge C \ln(n)$ for some constant $C$ that is independent of the jump size $k$. This is different from using the level-based methods for the whole process, as done in the proof of Theorem~\ref{thm:oldjump}, which would require an offspring population size at least logarithmic in the envisaged runtime, hence here $\Omega(\log n^k) = O(k \log n)$, which is super-logarithmic when $k$ is super-constant. 

While our arguments exploit some characteristics of the jump functions, we are optimistic that they can be employed for other problems as well, in particular, when the optimization process typically contains one step that is more difficult than the remaining optimization.

\section{Preliminaries}

In this section, we define the algorithm and the optimization problem regarded in this paper together with the most relevant works on these.

\subsection{The \mclea}\label{ssec:algo}

The \mclea for the maximization of pseudo-Boolean functions $f : \{0,1\}^n \to \R$ is made precise in Algorithm~\ref{alg:algo}. It is a simple non-elitist algorithm working with a parent population of size $\mu$ and an offspring population of size $\lambda \ge \mu$. Here and in the remainder by a population we mean a multiset of individuals (elements from the search space $\{0,1\}^n$). Each offspring is generated by selecting a random parent (independently and with replacement) from the parent population and mutating it via standard bit mutation, that is, by flipping each bit independently with probability $1/n$.\footnote{To ease the presentation, we only consider the standard mutation rate $1/n$, but we are confident that our results in an analogous fashion hold for general mutation rates $\chi/n$, $\chi$ a constant. Previous works have shown that the constant $\chi$ has an influence (again by constant factors) on where the boundary between the ``imitating elitism'' and ``no efficient progress'' regimes is located. Since our result is that the \mclea for no realistic parameter settings beats the \mplea, we do not expect that a constant factor change of the mutation rate leads to substantially different findings.} The next parent population consists of those $\mu$ offspring which have the highest fitness (breaking ties arbitrarily). 
	
\begin{algorithm2e}%
	Initialize $P_0$ with $\mu$ individuals chosen independently and uniformly at random from $\{0,1\}^n$\;
	\For{$t = 1, 2, \ldots$}{
    \For{$i \in [1..\lambda]$}{
      Select $x_i \in P_{t-1}$ uniformly at random\;
      Generate $y_i$ from $x_i$ via standard bit mutation\;
      }
    Select $P_t$ from the multi-set $\{y_1, \ldots, y_\lambda\}$ by choosing $\mu$ individuals of highest $f$-value (breaking ties arbitrarily)\;
  }
\caption{The \mclea to maximize a function $f : \{0,1\}^n \to \R$.}
\label{alg:algo}
\end{algorithm2e}

The \mplea, to which we compare the \mclea, differs from the \mclea only in the selection of the next generation. Whereas the \mclea selects the next generation only from the offspring population (comma selection), the \mplea selects it from the parent and offspring population (plus selection). In other words, to obtain the \mplea from Algorithm~\ref{alg:algo}, we only have to replace the selection by ``select $P_t$ from the multi-set $P_{t-1} \cup \{y_1, \ldots, y_\lambda\}$ by choosing $\mu$ best individuals (breaking ties arbitrarily)''. Often, the tie breaking is done by giving preference to offspring, but for all our purposes there is no difference.

When talking about the performance of the \mclea or the \mplea, as usual in runtime analysis~\cite{AugerD11,NeumannW10,Jansen13,DoerrN20}, we count the number of fitness evaluations until for the first time an optimal solution is evaluated. We assume that each individual is evaluated immediately after being generated. Consequently, if an optimum is generated in iteration $t$, then the runtime $T$ satisfies
\begin{equation}\label{eq:runtime}
 \mu + (t-1)\lambda + 1 \le T \le \mu + t \lambda.
\end{equation} 

Since we described the most important results on the \mclea already in Section~\ref{ssec:nonel}, let us briefly mention the most relevant results for the \mplea. Again, due to the difficulties in analyzing population-based algorithms, not too much is known. The runtimes of the \oplea, among others on \onemax and \leadingones, were first analyzed in~\cite{JansenJW05}. The asymptotically tight runtime on \onemax for all polynomial $\lambda$ was determined in~\cite{DoerrK15}, together with an analysis on general linear functions. In~\cite{Witt06}, the runtime of the \mpoea on \onemax and \leadingones, among others, was studied. The runtime of the \mplea with both non-trivial parent and offspring population sizes on the \onemax function was determined in~\cite{AntipovDFH18}.

\subsection{The Jump Function Class}\label{sec:jump}

To define the jump functions, we first recall that the $n$-dimensional \emph{$\onemax$ function} is defined by 
\[\OM(x) = \|x\|_1 = \sum_{i=1}^n x_i\] 
for all $x \in \{0,1\}^n$

Now the $n$-dimensional \emph{jump function} with \emph{jump parameter (jump size)} $k \in [1..n]$ is defined by
\[
\jump_{nk}(x) = 
\begin{cases}
\|x\|_1+k & \mbox{if $\|x\|_1 \in [0..n-k] \cup \{n\}$,}\\
n - \|x\|_1 & \mbox{if $\|x\|_1 \in [n-k+1\, ..\, n-1]$}.
\end{cases}
\]
Hence for $k = 1$, we have a fitness landscape isomorphic to the one of $\onemax$, but for larger values of $k$ there is a fitness valley (``gap'')
\[G_{nk} \coloneqq \{x \in \{0,1\}^n \mid n-k < \|x\|_1 < n\}\] 
consisting of the $k-1$ highest sub-optimal fitness levels of the \onemax function. This valley is hard to cross for evolutionary algorithms using standard bit mutation. When using the common mutation rate $\frac 1n$, the probability to generate the optimum from a parent on the local optimum is only $p_k := (1-\frac 1n)^{n-k} n^{-k} < n^{-k}$. For this reason, e.g., the classic $(\mu+\lambda)$ EA has a runtime of at least $n^k$ when $k$ is not excessively large. This was proven formally for the \oea in the classic paper~\cite{DrosteJW02}, but the argument can easily be extended to all $(\mu+\lambda)$ EAs (as we do now for reasons of completeness). We also prove an upper bound, which will later turn out to agree with our lower bound for the \mclea for large ranges of the parameters. 

\begin{theorem}\label{thm:plus}
  Let $\mu, \lambda \in \Z_{\ge 1}$. Let $n \in \Z_{\ge 2}$ and $k \in [2..n]$. Let $p_k := (1-\frac 1n)^{n-k} n^{-k}$. Let $T$ denote the runtime, measured by the number of fitness evaluations until the optimum is found, of the $(\mu+\lambda)$ EA on the $\jump_{nk}$ function.
\begin{enumerate}
\item Let $h(n):= \sqrt{2n \log(\mu n)})$. If $k \le \frac n2 - h(n)$, then 
\[E[T] \ge  \left(1-\frac 1n\right) \left(\mu + \frac{1}{p_k}\right),\] 
otherwise $E[T] \ge (1 - \frac 1n) \left(\mu + \frac{1}{p_{k'}}\right)$ with $k' := \frac n2 - h(n)$. 
\item $E[T] \le \frac{1}{p_k} + O\left(n \log n + n\mu + n \lambda \frac{\log^+ \log^+ (\lambda/\mu)}{\log^+ (\lambda/\mu)} + (\mu+\lambda) \log \mu\right)$, where we write $\log^+ x := \max\{1,\ln x\}$ for all $x > 0$. If ${\mu \le \lambda}$, $\lambda = \exp(O(n))$, and $\lambda = o(\frac 1{n p_k})$, then $E[T] \le (1+o(1)) \frac{1}{p_k}$.
\end{enumerate}
\end{theorem}

\begin{proof}
  To cover both cases, let $k' = \min\{k, \frac n2 - h(n)\}$. Using the additive Chernoff bound (Theorem~\ref{tprobchernoffadditive01}) and a union bound, we see that with probability at least 
  \[1 - \mu \exp\left(-\frac{(\frac n2 -k')^2}{2n}\right) \ge 1 - \mu \exp\left(-\frac{h(n)^2}{2n}\right) \ge 1 - \frac 1n\]
  all $\mu$ initial individuals $x$ satisfy $\OM(x) \le n - k'$. Conditioning on this, in the remaining run all individuals that are taken into the parent population also satisfy $\OM(x) \le n - k'$ (unless they are the optimum). Consequently, for an offspring to become the first optimum sampled, there is a unique set of $\ell \ge k'$ bits in the parent that need to be flipped (and the other bits may not be flipped). The probability for this event is $(1-\frac 1n)^{n-\ell} (\frac 1n)^\ell \le (1-\frac 1n)^{n-k'} (\frac 1n)^{k'} = p_{k'}$. Hence the time until this happens is stochastically dominated (see, e.g.,~\cite{Doerr19tcs}) by a geometric distribution with success probability $p_{k'}$, which has an expectation of $\frac 1{p_{k'}}$. Together with the $\mu$ initial fitness evaluations, this shows the lower bound. 
  
  For the upper bound, we use a recent analysis of the runtime of the \mplea on \onemax. In~\cite{AntipovDFH18}, it was shown that the \mplea finds the optimum of \onemax in an expected number of 
  \[O\left(n \log n + n\mu + n \lambda \frac{\log^+ \log^+ (\lambda/\mu)}{\log^+ (\lambda/\mu)}\right)\]
fitness evaluations (the result is stated in terms of iterations in~\cite{AntipovDFH18}, but with~\eqref{eq:runtime} one immediately obtains the form above). It is easy to see from the proof in~\cite{AntipovDFH18} that this bound also holds for the expected time until the \mplea optimizing any jump function as found an individual on the local optimum (if it has not found the optimum before). 

What cannot be taken immediately from the previous work is remainder of the runtime analysis. In particular, since generating the optimum from the local optimum is more difficult than generating an individual on the next \onemax fitness level, we need a larger number of individuals on the local optimum before we have a reasonable chance of making progress. Since we mostly aim at a good upper bound in the regime where $\mu$ and $\lambda$ are not excessively large, we allow for the time until the whole population is on the local optimum or better. By~\cite[Lemma~2]{Sudholt09}, this takeover time is $O((\mu+\lambda) \log \mu)$ fitness evaluations (or the optimum is found earlier). From this point on, any further individual has a probability of exactly $p_k$ of being the optimum, giving an additional $\frac 1 {p_k}$ term for the runtime bound. This shows the general upper bound. If $\mu \le \lambda$, $\lambda = \exp(O(n))$ and $\lambda = o(\frac{1}{np_k})$, then $(\mu+\lambda)\log \mu = O(\lambda \log \lambda) = O(n \lambda) = o(1/p_k)$. For similar reasons, the expressions $n\mu$ and $n \lambda \frac{\log^+ \log^+ (\lambda/\mu)}{\log^+ (\lambda/\mu)}$ are of lower order. Since $k \ge 2$, we have $p_k = \Omega(n^2)$, and thus also the $n \log n$ expression is of lower order. 
\end{proof}

By using larger mutation rates or a heavy-tailed mutation operator, the runtime of the \oea can be improved by a factor of $k^{\Theta(k)}$~\cite{DoerrLMN17}, but the runtime remains $\Omega(n^k)$ for $k$ constant. 

Asymptotically better runtimes can be achieved when using crossover, though this is not as easy as one might expect. The first work in this direction~\cite{JansenW02}, among other results, showed that a simple $(\mu+1)$ genetic algorithm using uniform crossover with rate $p_c = O(\frac{1}{kn})$ has an $O(\mu n^2 k^3 + 2^{2k} p_c^{-1})$ runtime when the population size is at least $\mu = \Omega(k \log n)$. A~shortcoming of this result, as noted by the authors, is that it only applies to uncommonly small crossover rates. Using a different algorithm that first applies crossover and then mutation, a runtime of $O(n^{k-1} \log n)$ was achieved by Dang et al.~\cite[Theorem~2]{DangFKKLOSS18}. For $k \ge 3$, the logarithmic factor in the runtime can be removed by using a higher mutation rate. With additional diversity mechanisms, the runtime can be further reduced down to $O(n \log n + 4^k)$, see~\cite{DangFKKLOSS16}. The \opllga can optimize $\jump_k$ in time $O(n^{(k+1)/2} k^{-\Omega(k)})$~\cite{AntipovDK20}.

With a three-parent majority vote crossover, among other results, a runtime of $O(n \log n)$ could be obtained via a suitable island model for all $k = O(n^{1/2 - \eps})$~\cite{FriedrichKKNNS16}. A different voting algorithm also giving an $O(n \log n)$ runtime was proposed in~\cite{RoweA19}. Via a hybrid genetic algorithm using as variation operators only local search and a deterministic voting crossover, an $O(n)$ runtime was shown in~\cite{WhitleyVHM18}. 

Runtimes of $O\left(n \binom{n}{k}\right)$ and $O\left(k \log(n) \binom{n}{k}\right)$ were shown for the $(1+1)$~IA$^{\mathrm{hyp}}$ and the $(1+1)$ Fast-IA artificial immune systems, respectively~\cite{CorusOY17,CorusOY18fast}. In~\cite{LissovoiOW19}, the runtime of a hyper-heuristic switching between elitist and non-elitist selection was studied. The lower bound of order $\Omega(n \log n) + \exp(\Omega(k))$ and the upper bound of order $O(n^{2k-1}/k)$, however, are too far apart to indicate an advantage or a disadvantage over most classic algorithms. In this work, it is further stated that the Metropolis algorithm (using the 1-bit neighborhood) has an $\exp(\Omega(n))$ runtime on jump functions.  

Without diversity mechanisms and non-standard operators, the compact genetic algorithm, a simple estimation-of-distribution algorithm, has a runtime  of $O(n \log n + 2^{O(k)})$~\cite{HasenohrlS18,Doerr19gecco}.

\section{Technical Tools}

In this section, we collect a few technical tools that will be used in our proofs. All but the last one, an elementary non-asymptotic lower bound for the probability to generate an offspring with equal \onemax fitness, are standard tools in the field. 

Let $X$ be a binomially distributed random variable with parameters $n$ and $p$, that is, $X = \sum_{i=1}^n X_i$ with independent $X_i$ satisfying $\Pr[X_i = 1] = p$ and $\Pr[X_i = 0] = 1-p$. Since $X$ is a sum of independent binary random variables, Chernoff bounds can be used to bound its deviation from the expectation. However, the following elementary estimate also does a good job. This estimate appears to be well-known (e.g., it was used in~\cite{JansenJW05} without proof or reference). Elementary proofs can be found in~\cite[Lemma~3]{GiessenW17} or~\cite[Lemma~1.10.37]{Doerr20bookchapter}.

\begin{lemma}\label{lprobbino}
  Let $X \sim \Bin(n,p)$. Let $k \in [0..n]$. Then \[\Pr[X \ge k] \le \binom{n}{k} p^k.\]
\end{lemma}

The following \emph{additive Chernoff bound} from Hoeffding~\cite{Hoeffding63}, also to be found, e.g., in~\cite[Theorem~1.10.7]{Doerr20bookchapter}, provides a different way to estimate the probability that a binomial random variable and, in fact, any sum of bounded independent random variables exceeds its expectation.

\begin{theorem}\label{tprobchernoffadditive01}
  Let $X_1, \ldots, X_n$ be independent random variables taking values in $[0,1]$. Let $X = \sum_{i = 1}^n X_i$. Then for all $\lambda \ge 0$,  
  \begin{align*}
    \Pr[X \ge E[X] + \lambda] & \le \exp\bigg(-\frac{2\lambda^2}{n}\bigg).
  \end{align*}
\end{theorem}

As part of our \emph{additive drift with rescaling} lower bound proof strategy, we need the following additive drift theorem of He and Yao~\cite{HeY01}, see also~\cite[Theorem~2.3.1]{Lengler20bookchapter}, which allows to translate a uniform upper bound on an expected progress into a lower bound on the expected time to reach a target.
 
\begin{theorem}\label{tadddrift}
  Let $S \subseteq \R_{\ge 0}$ be finite and $0 \in S$. Let $X_0, X_1, \ldots$ be a stochastic process taking values in $S$. Let $\delta > 0$. Let $T = \inf\{t \ge 0 \mid X_t = 0\}$. If for all $t \ge 0$ and all $s \in S \setminus \{0\}$ we have $E[X_t - X_{t+1} \mid X_t = s] \le \delta$, then $E[T] \ge \frac{E[X_0]}{\delta}$.
\end{theorem}

Finally, we shall use occasionally the following lower bound on the probability that standard bit mutation creates from a parent $x$ with $\OM(x) < n$ an offspring with equal \onemax-value. The main difference to the usual estimate $(1-\frac 1n)^n = (1-o(1)) \frac 1e$, which is the probability to recreate the parent, is that our lower bound is exactly $\frac 1e$, which avoids having to deal with asymptotic notation. 

\begin{lemma}\label{lem:gleich}
  Let $x \in \{0,1\}^n$ with $0 < \OM(x) < n$. Let $y$ be obtained from $x$ via standard bit mutation with mutation rate $\frac 1n$. Then ${\Pr[\OM(y) = \OM(x)] \ge \frac 1e}$.
\end{lemma}

\begin{proof}
  Let $k := \OM(x)$. For $y$ to have this same $\OM$-value, it suffices that either no bit in $x$ is flipped or that exactly one zero-bit and exactly one one-bit are flipped. The probability for this event is $(1-\frac 1n)^n + \frac{k(n-k)}{n^2}(1-\frac 1n)^{n-2} \ge (1-\frac 1n)^n + \frac{n-1}{n^2}(1-\frac 1n)^{n-2} = (1-\frac 1n)^{n-1} \ge \frac 1e$.
\end{proof}

\section{A Lower Bound for the Runtime of the \mclea on Jump Functions}\label{sec:lower}

In this section, we prove our main result, a lower bound for the runtime of the \mclea on jump functions which shows that for a large range of parameter values, the \mclea cannot even gain a constant factor speed-up over the \mplea. With its $\Omega(n^k)$ order of magnitude, our result improves significantly over the only previous result on this problem, the $\exp(\Omega(k))$ lower bound in~\cite{Lehre10} (Theorem~\ref{thm:oldjumpLB} in this work).

Before stating the precise result, we quickly discuss two situations which, in the light of previous results, do not appear overly interesting and for which we therefore did not make an effort to fully cover them by our result.
\begin{itemize}
\item When $\lambda \le (1-\eps) e \mu$ for an arbitrarily small constant $\eps \ge 0$ and $\lambda$ is at most polynomial in $n$, the results of Lehre~\cite{Lehre10} imply that the \mclea has an exponential runtime on any function with a polynomial number of optima (and consequently, also on jump functions). We guess that the restriction to polynomial-size $\lambda$ was made in~\cite[Corollary~1]{Lehre10} only for reasons of mathematical convenience (together with the fact that super-polynomial population sizes raise some doubts on the implementability and practicability of the algorithm). We do not see any reason why Lehre's result, at least in the case of the \mclea, should not be true for any value of $\lambda$ (possibly with a sub-exponential number of iterations, but still an exponential number of fitness evaluations).
\item Rowe and Sudholt~\cite[Theorem~10]{RoweS14} showed that for all constants $\eps > 0$ the $\oclea$ with population size $\lambda \le (1-\eps) \log_{\frac{e}{e-1}} n$ has an expected optimization time of at least $\exp(\Omega(n^{\eps/2}))$ on any function $f : \{0,1\}^n \to \R$ with a unique optimum. From inspecting the proof given in~\cite{RoweS14}, we strongly believe that the same result also holds for the $\mclea$. Since this is not central to our work, we do not give a rigorous proof. Our main argument would be that the runtime of the $\mclea$ on a jump function is at least (in the strong sense of stochastic domination) its runtime on \onemax. This follows from a coupling arguments similar to the one given in~\cite[Proof of Theorem~23]{Doerr19tcs}. More precisely, when comparing how offspring are selected in the \jump and in the \onemax process, we can construct a coupling such that the parent individuals in the \jump process always are not closer to the optimum than in the \onemax process. Now comparing how the \oclea and the \mclea optimize \onemax, we see that the \mclea may select worse parent individuals than the \oclea, which generate (in the stochastic domination sense) worse offspring, leading to a larger optimization time. As said, we do not declare this a complete proof, but since the case of small population sizes might generally not be too interesting and since our not fully rigorous analysis indicates that an interesting performance of the \mclea is not to be expected here, we refrain from giving a complete analysis. 
\end{itemize}

Let us declare the parameter settings just discussed as not so interesting since previous works show or strongly indicate that the \mclea is highly inefficient on any objective function with unique optimum. Let us further declare exponential population sizes as not so interesting (mostly for reasons of implementability, but also because Lemma~\ref{lem:large} will show that they imply exponential runtimes). With this language, our following result shows that the runtime of the \mclea on jump functions with jump size $k \le 0.1n$ for all interesting parameter choices is, apart from lower order terms, at least the one of the \mplea. For $k > 0.1 n$, this runtime is at least $n^{\Omega(n)}$.

\begin{theorem}\label{thm:lower}
  Let $c \le 0.1$ and $C$ be large enough such that $(4c)^{C/2} \le e^{-2}$. Let $n \ge \frac 2c$. Let $C \ln(n) \le \lambda \le \frac 23 \exp(0.16n)$ and $\mu \le \frac \lambda 2$. Let $c' = \frac 1e + c$ and $h(n,\lambda) : = \exp(-\frac{(1-2c')^2}{2} \lambda) + \frac {2n-1}{n^2 - n}$. Let $k \in [2..n]$ and $p_k := (1 - \tfrac 1n)^{n-k} n^{-k}$. 
  
  If $k \le cn$, then the expected runtime, measured by the number of fitness evaluations until the optimum is evaluated, of the \mclea on jump functions with jump size $k$ is at least 
  \[T_{k} := (1 - \exp(-0.16n)) \left(\mu + \left(1 - h(n,\lambda)\right) \tfrac{1}{p_k}\right) = (1 - o(1))(\mu+\tfrac 1 {p_k}),\]
  where the asymptotic expression is for $n \to \infty$ and $\lambda = \omega(1)$.
  
  For $k > cn$, the expected runtime is at least $T_{\lfloor cn \rfloor}$.
\end{theorem}

We phrased our result in the above form since we felt that it captures best the most interesting aspect, namely a runtime of essentially $\frac 1 {p_k}$ when $k \le 0.1n$ and $\lambda = \Omega(\log n)$ suitably large. Since our result is non-asymptotic, both $c$ and $C$ do not have to be constants. Hence if we are interested in the smallest possible value for $\lambda$ that gives an $(1-o(1)) \frac 1 {p_k}$ runtime, then by taking $c = \frac kn$ and $C = 4 / \ln(n/(4k))$, we obtain the following result.

\begin{corollary}\label{cor:lower}
  Let $k \ge 2$. Let $n \ge 10k$ and \[\frac{4}{\ln(\frac{n}{4k})} \ln(n) \le \lambda \le \tfrac 23 \exp(0.16n).\] Let $\mu \le \frac \lambda 2$. Let $c' = \frac 1e + \frac kn$. With $h(n,\lambda) : = \exp(-\frac{(1-2c')^2}{2} \lambda) + \frac {2n-1}{n^2 - n}$ and $p_k := (1 - \tfrac 1n)^{n-k} n^{-k}$, the expected runtime of the $\mclea$ on $\jump_{nk}$ is at least 
  \[T_{k} := (1 - \exp(-0.16n)) \left(\mu + \left(1 - h(n,\lambda)\right) \tfrac{1}{p_k}\right) = (1 - o(1))(\mu+\tfrac 1 {p_k}),\]
  where the asymptotic expression holds for $n \to \infty$ and $\lambda = \omega(1)$.
  
  In particular, if $k = O(n^{1-\eps})$ for a constant $\eps>0$, then it suffices to have $\lambda = \omega(1)$  for the lower bound $(1 - o(1))(\mu+\tfrac 1 {p_k})$ to hold. 
\end{corollary}


Before giving the precise proof of Theorem~\ref{thm:lower}, let us brief{}ly explain the main ideas. As discussed earlier, this proof is an example for proving lower bounds by applying the additive drift theorem to a suitable rescaling of a natural potential function. As this work shows, this method can give very tight lower bounds, different from, say, negative drift theorems. 

The heart, and art, of this method is defining a suitable potential function. The observation that the difficult part of the optimization process is traversing the region $\{x \in \{0,1\}^n \mid \OM(x) \in [n-k..n]\}$ together with the fact that the lower bound given by the additive drift theorem depends on the difference in potential of starting point and target suggested to us the following potential function. For a population $P$, let $\OM(P)$ denote the maximum \onemax-value in the population. For $\OM(P) > n-k$, the potential of $P$ will essentially be $\min\{n^{\OM(P)-(n-k)},\frac{1}{\lambda p_k}\}$. All other populations have a potential of zero. This definition gives the desired large potential range of $\frac{1}{\lambda p_k}$ and, after proving that the expected potential gain is at most one and the initial potential is zero with high probability, gives the desired lower bound on the runtime. The proof below gives more details, including the precise definition of the potential, which is minimally different from this simplified description.

Since the classic definition of the runtime of an evolutionary algorithm, the number of fitness evaluations until the optimum is evaluated, implies that we usually lose a number of $\lambda - 1$ fitness evaluations when translating iterations into fitness evaluations (since we usually cannot rule out that the optimum is sampled as the first individual generated in the iteration which finds the optimum, see~\eqref{eq:runtime}), we add a short argument to the proof of Theorem~\ref{thm:lower} to gain an extra $\lambda$ fitness evaluations. The main observation is that both the $\mu$ initial individuals and the $\lambda$ offspring of the first generation are uniformly distributed in the search space. This makes them very unlikely to be the optimum, and this argument (with some finetuning) allows us to start the drift argument from the second iteration on.

We now give the complete proof of our lower bound result.

\begin{proof}[Proof of Theorem~\ref{thm:lower}]
  For a unified proof for the two cases that $k$ is at most $cn$ or greater than $cn$, let us denote by $k'$ the jump size and recall that this can be a function of $n$. Let $k := \min\{k', \lfloor cn \rfloor\}$. Assume that $n$ and $\lambda$ are large enough so that $h(n,\lambda) < 1$, as otherwise there is nothing to show. 
  
Let $\gmax = (1 - h(n,\lambda))\frac 1 {\lambda p_k}$. Note that by our assumption, $\gmax > 0$. Also, we easily see that $\gmax \le n^k$: If $\lambda \ge 3$, then $\gmax \le \frac{1}{\lambda p_k} \le \frac{1}{\lambda n^{-k} / e} \le n^k$; however, if $\lambda = 2$, the only other option leaving us with a positive integral $\mu$, then $h(n,\lambda) \ge \exp(- (1 - 2/e)^2) > 0.93$ and again $\gmax = (1-h(n,\lambda)) \frac 1 {\lambda p_k} \le 0.07 \frac{1}{\lambda n^{-k} / e} \le n^k$.

For all $L \in [1..k]$, let $g(L) := \min\{n^L, \gmax\}$, and let $g(0) := 0$. 
   Let $k^*$ be the smallest integer in $[1..k]$ such that $g(k^*) = \gmax$. Note that $k^*$ is well defined since $\gmax \le n^k$.
  
  For all individuals $x \in \{0,1\}^n$, denote by 
  \begin{align*}
  \OM(x) &\coloneqq \sum_{i=1}^n x_i \in [0..n] \text{ its \onemax-value,}\\
  \ell(x) &\coloneqq \max\{0,\OM(x) - (n-k)\} \in [0..k],\\
  g(x) &\coloneqq g(\ell(x)) \in \{0, n, n^2, \ldots, n^k, \gmax\} \cap [0,\gmax].
  \end{align*}
  For a population $P$, that is, a multiset of individuals, we write 
  \begin{align*}
  \OM(P) &\coloneqq \max\{\OM(x) \mid x \in P\},\\
  \ell(P) &\coloneqq \max\{\ell(x) \mid x \in P\} = \max\{0,\OM(P) - (n-k)\},\\
  g(P) &\coloneqq \max\{g(x) \mid x \in P\} = g(\ell(P)).
  \end{align*}
  We use $g(P)$ as a measure for the quality of the current population of the algorithm. We shall argue that we typically start with $g(P) = 0$ and that one iteration increases $g(P)$ in expectation by at most $1$. Since we have $g(P) = \gmax$ if $P$ contains the optimum, the additive drift theorem yields that it takes an expected number of at least $g(k^*)$ iterations to find the optimum. Let us make these arguments precise.
  
We first show that with high probability both the initial population $P_0$ and the first offspring population (and, consequently, also $P_1$) contain no individual $x$ with $\OM(x) > n-k$. For this, we first observe that trivially the random initial individuals are uniformly distributed in $\{0,1\}^n$. We recall that if $x \in \{0,1\}^n$ is uniformly distributed, then this is equivalent to saying that the random variables $x_1, x_2, \dots, x_n$ are independent and uniformly distributed in $\{0,1\}$. Interestingly, also the individuals of the first offspring population are uniformly distributed, since they are generated from taking a random individual by flipping each bit independently with probability $\frac 1n$. Consequently, all their bits are independent and uniformly distributed in $\{0,1\}$ (the different offspring are not independent, but this independence is not needed in the following). By the additive Chernoff bound (Theorem~\ref{tprobchernoffadditive01}), the probability that a random individual $x$ satisfies $\OM(x) > (1-c)n$ is at most $\exp(-2 (0.5-c)^2 n^2 / n) = \exp(-0.32 n)$ by our choice of $c$. Since $\mu \le \frac \lambda 2$ and $\lambda \le \frac 23 \exp(0.16n)$, a simple union bound over the $\mu$ initial individuals and the $\lambda$ offspring shows that with probability at least $1 - (\mu+\lambda)\exp(-0.32n) \ge 1 - \exp(-0.16n)$, none of these individuals $x$ satisfies $\OM(x) > n-k$. In this case, the population $P_1$ satisfies $g(P_1) = 0$ and up to this point, the algorithm has used $\mu+\lambda$ fitness evaluations without ever sampling the optimum.

  We now analyze the effect of one iteration. Let $P$ be a population of $\mu$ elements that does not contain the optimum. Let $P'$ be the (random) offspring population generated in one iteration of the \mclea started with $P$, and let $P''$ be the next parent population, that is, the random population generated from selecting $\mu$ best individuals from $P'$. 
  
  Let $i \coloneqq \OM(P)$ and $j > \max\{i,n-k\}$. Hence if $\OM(P'') = j$, then $\ell(P'') > \ell(P)$, and in the case that $i < n - k + k^*$ we have made a true progress with respect to our measure $g(P)$. 
  For this reason, we now compute the probability that $\OM(P'') = j$. 
  
  We consider first the case $j = n$. Note that here $\OM(P') = j$ implies $\OM(P'')=j$. Let $x$ be an element of $P$ and $y$ be an offspring generated from $x$ by mutation. Then $\Pr[y = (1, \ldots, 1)] = (1 - \frac 1n)^{\OM(x)} n^{-(n-\OM(x))} \le (1 - \frac 1n)^{i} n^{-(n-i)}$, using that $n \ge \frac 2c \ge 2$. By a union bound over the $\lambda$ offspring, we have $\OM(P') = \OM(P'') = n$ with probability at most 
  \begin{equation}
  \Pr[\OM(P'') = n] \le \lambda (1 - \tfrac 1n)^{i} n^{-(n-i)}.\label{eq:maxlevel}
  \end{equation}

  Let now $j < n$. By the definitions of the \mclea and the jump functions, we have $\OM(P'') = j$ only if $P'$ contains at least $\lambda - \mu + 1$ individuals $y$ with $\OM(y) \in [j..n-1]$. To obtain an upper bound for $\Pr[\OM(P'') = j]$ we regard the (slightly larger) event $\calE$ that $P'$ contains at least $\lambda - \mu + 1$ individuals $y$ with $\OM(y)  \in [j..n]$.
  
  To analyze this event $\calE$, let again $x \in P$ and $y$ be a mutation-offspring generated from~$x$. By a natural domination argument~\cite[Lemma~6.1]{Witt13}, the probability of the event $\OM(y) \ge j$ does not decrease if we increase $\OM(x)$. For this reason, let us assume that $\OM(x) = \max\{i,n-k\} =: \tilde\imath$. Now for $\OM(y) \ge j$ to hold, at least $j-\tilde\imath$ of the $n-\tilde\imath$ zero-bits in $x$ have to be flipped. The number of flipped zero-bits follows a binomial distribution with parameters $n-\tilde\imath$ and $\frac 1n$. By Lemma~\ref{lprobbino}, we have 
  \begin{align*}
  \Pr[\OM(y) \ge j] &\le \binom{n-\tilde\imath}{j-\tilde\imath} n^{-(j-\tilde\imath)} 
  \le \left(\frac {n-\tilde\imath}n\right)^{j-\tilde\imath} \eqqcolon p_{\tilde\imath j}.
  \end{align*}
  Since the individuals of the offspring population $P'$ are generated independently, the number of offspring $y$ with $\OM(y) \in [j..n]$ is binomially distributed with parameters $\lambda$ and some number $p' \le p_{\tilde\imath j}$. Using again Lemma~\ref{lprobbino}, we see that  
  \begin{align}
  \Pr[\OM(P'') = j] &\le \Pr[\calE] \nonumber\\
  &\le \binom{\lambda}{\lambda - \mu + 1} (p')^{\lambda - \mu + 1} \le 2^\lambda \left(\left(\frac {n-\tilde\imath }n\right)^{j-\tilde\imath }\right)^{\lambda - \mu + 1} \nonumber\\
  & \le 2^\lambda \left(\left(\frac kn\right)^{j-\tilde\imath}\right)^{\lambda/2} \le \left(\frac{4k}{n}\right)^{(j-\tilde\imath)\lambda/2} \nonumber\\
  &\le \left((4c)^{C/2}\right)^{\ln(n) (j-\tilde\imath)} \le n^{-2(j-\tilde\imath)},\label{eq:levelgain}
  \end{align}
  where we used the estimates $\binom{\lambda}{\lambda - \mu + 1} \le 2^\lambda$ and our assumptions $\mu \le \frac \lambda 2$ and $(4c)^{C/2} \le e^{-2}$.

So far we have computed that it is difficult to strictly increase $\OM(P)$ (once $\OM(P)$ is above $n-k$). Using a similar reasoning, we now show that also the probability of the event $\OM(P) = \OM(P'')$ is small (when $i \coloneqq \OM(P) > n-k$). Again, for this event it is necessary that at least $\lambda - \mu + 1$ offspring $y$ satisfy $\OM(y) \in [i..n]$. Let $y$ be an offspring generated from a parent $y \in P$. As above, using a domination argument we can assume that $\OM(x) = \OM(P) = i$. For such an $x$, we have $\OM(y) \ge i$ only if either no bit at all flips or if at least one zero-bit is flipped. Hence $\Pr[\OM(y) \ge i] \le (1-\frac 1n)^n + p_{i,i+1} \le \frac 1e + \frac kn \le \frac 1e + c = c' < 0.5$. 
Denoting by $X$ the number of offspring $y$ with $\OM(y) \ge i$, we have $E[X] \le c' \lambda$. Using the additive Chernoff bound (Theorem~\ref{tprobchernoffadditive01}) rather than Lemma~\ref{lprobbino} and $\mu \le \tfrac 12 \lambda$, 
we compute  
\begin{align}
  \Pr[\OM(P'') = i] &\le \Pr[X \ge \lambda-\mu+1]\nonumber\\
  &\le \Pr[X \ge \tfrac 12 \lambda] \le \Pr[X \ge E[X] + (0.5-c') \lambda]\nonumber\\
  &\le \exp\left(- 2 \frac{((0.5-c') \lambda)^2}{\lambda}\right) = \exp\left(-\frac{(1-2c')^2}{2} \lambda\right).\label{eq:staylevel}
\end{align}

We are now ready to compute the expected progress of $g(P)$ in one iteration. Let first $\ell(P) = 0$ and thus $g(P)=0$. Since $\Pr[\ell(P'') = L] \le n^{-2L}$ for all $L \in [1..k-1]$ by~\eqref{eq:levelgain} and $\Pr[\ell(P'') = k] \le \lambda p_k$ by~\eqref{eq:maxlevel}, we have 
\begin{align*}
  E[g(P'')] &\le 1 \cdot g(0) + \sum_{L = 1}^{k-1} n^{-2L}  \cdot g(L) + \lambda p_k \cdot \gmax\\
  &\le 0 + \sum_{L = 1}^{k-1} n^{-L} + 1 - h(n,\lambda) \le \frac{1}{n-1} + 1 - h(n,\lambda) \le 1
\end{align*} 
by the choice of $\gmax$ and $h$. Consequently, $E[g(P'') - g(P)] \le 1$.

For $\ell \coloneqq \ell(P) > 0$, the probability to reach $\ell(P'') = k$ is larger, however, we profit from the fact that we can reduce the potential by having ${\ell(P'') < \ell}$. Since there is nothing to show when $g(P)$ is already at the maximal value $\gmax$, let us assume that $\ell < k^*$. Now equations~\eqref{eq:maxlevel}, \eqref{eq:levelgain}, and \eqref{eq:staylevel} give
\begin{align*}
  &\Pr[\ell(P'') = k] \le \lambda (1-\tfrac 1n)^{n-(k-\ell)} n^{-(k-\ell)} \le \lambda (1-\tfrac 1n)^{n-k} n^{-(k-\ell)},\\
  &\Pr[\ell(P'') = L] \le n^{-2(L-\ell)}, L \in [\ell+1..k-1],\\
  &\Pr[\ell(P'') = \ell] \le \exp\left(-\frac{(1-2c')^2}{2} \lambda\right).
\end{align*}
With these estimates, we compute
\begin{align*}
  E[g(P'')] 
  &\le \Pr[\ell(P'') = k] \, \gmax + \sum_{L=\ell+1}^{k-1} \Pr[\ell(P'') = L] \, g(L) \\
  &\quad + \Pr[\ell(P'') = \ell] \, g(\ell) + 1 \cdot g(\ell-1)\\
  &\le \lambda (1-\tfrac 1n)^{n-k} n^{-(k-\ell)} \cdot (1-h(n,\lambda)) \tfrac 1\lambda (1-\tfrac 1n)^{-(n-k)} n^k \\
  &\quad + \sum_{L=\ell+1}^{\infty} n^{-2(L-\ell)} n^L + \exp\left(-\frac{(1-2c')^2}{2} \lambda\right) n^\ell + n^{\ell-1}\\
  & = g(P) \left(1 - h(n,\lambda) + \frac{1}{n-1} + \exp\left(-\frac{(1-2c')^2}{2} \lambda\right) + \frac 1n \right)\\
  & = g(P)
\end{align*}
by our choice of $\gmax$ and $h$ as well as our assumption that $g(P) < \gmax$. Consequently, again we have $E[g(P'') - g(P)] \le 1$.

Assuming that the population $P_1$ satisfies $g(P_1) = 0$, we can now apply the additive drift theorem (Theorem~\ref{tadddrift}) as follows. As before, let $P_t$ denote the population at the end of iteration $t$. For all $t \ge 0$, let $X_t = \gmax - g(P_{t+1})$. Then $X_0 = \gmax$ and $E[X_{t} - X_{t+1} \mid X_t > 0] \le 1$ for all $t \ge 0$. Consequently, the additive drift theorem (Theorem~\ref{tadddrift}) gives that $T := \min\{t \mid X_t = 0\}$ has an expectation of at least $\gmax$. By definition, $X_t = 0$ is equivalent to saying that $P_{t+1}$ contains the optimum. Recall that if optima are generated in some iteration $t$, then some remain in $P_t$. Hence $T+1$ is indeed the first iteration in which the optimum is generated.

If the optimum is found in some iteration $t \ge 1$, then the total number of fitness evaluations up to this event is at least $\mu + (t-1)\lambda + 1$, where the $\mu$ accounts for the initialization and the $-1$ and $+1$ for the fact that the optimum could be the first search point sampled in iteration $t$ (so that we cannot count the remaining offspring generated in the last iteration, see also~\eqref{eq:runtime}). This gives an expected optimization time of at least $\mu + (E[T+1]-1) \lambda + 1 = \mu + \gmax\lambda + 1$ in the case $g(P_1) = 0$.

Since we have $g(P_1)=0$ with probability $1 - \exp(-0.16n)$, the expected runtime is at least $(1 - \exp(-0.16n)) (\mu + \gmax \lambda + 1)$. Recalling that $\gmax = (1-h(n,\lambda)) \frac 1 {\lambda p_k}$, we have proven the theorem.
\end{proof}  

Since it might be useful in other applications, we now explicitly formulate our lower bound of, essentially, $\mu+\lambda$, which was observed in the proof above.

\begin{lemma}\label{lem:large}
  Let $f : \{0,1\}^n \to \R$. Assume that $f$ has at most $M$ global optima. Let $\mu, \lambda$ be positive integers. Consider the optimization of $f$ via the \mclea (and assume $\mu \le \lambda$ in this case) or the \mplea. Let $N \le \mu+\lambda$. Then with probability $1 - M N 2^{-n}$, the optimization time is larger than $N$. In particular, the expected optimization time is at least $(1 - M N 2^{-n}) (N+1) \ge \frac 14 \min\{\mu+\lambda, 2^n / M\}$. 
\end{lemma}

\begin{proof}
  As discussed in the proof of Theorem~\ref{thm:lower}, each of the first $\mu + \lambda$ individuals generated in a run of the \mclea (and the same applies to the \mplea) is uniformly distributed in $\{0,1\}^n$. Consequently, it is an optimum with probability at most $M 2^{-n}$. By a union bound over the first $N$ of these $\mu+\lambda$ individuals, the probability that one of them is optimal, is at most $N M 2^{-n}$. This gives the claims, where the last estimate follows from taking $N = \min\{\mu+\lambda,2^n / (2M)\}$.  
\end{proof}

\section{A Tight Upper Bound}

While our main target in this work was showing a lower bound that demonstrates that the \mclea has little advantage in leaving the local optima of the jump functions, we now also present an upper bound on the runtime. It shows that our lower bound for large parts of the parameter space is tight including the leading constant. This might be the first non-trivial upper bound for a non-elitist evolutionary algorithm that is tight including the leading constant. This result also shows that our way to exploit negative drift in the lower bound analysis, namely not via the classic negative drift theorems, but via additive drift applied to an exponential rescaling, can give very precise results, unlike the previously used methods.

We shall show the following result. 

\begin{theorem}\label{thm:upper}
  Let $K$ be a sufficiently large constant and $\lambda  \ge K \ln n$. Let $0 < \delta < 1$ be a constant and $\mu \le \frac{1}{(1+\delta)e} \lambda$. Let $k \in [2..n]$ and $p_k = (1-\frac 1n)^{n-k} n^{-k}$. Then the runtime $T$ of the \mclea on $\jump_{nk}$ satisfies \[E[T] \le \frac{\lambda}{1 - n^{-1/2}} \left( 8 Cn + 1 + 9 \sqrt{\frac{Cn}{p_k \lambda}} + \frac{8Cn}{p_k \lambda \lfloor n^{3/2} \rfloor} + \frac{1}{p_k \lambda}\right),\]
  where $C$ is a constant depending on $\delta$ only.\footnote{More precisely, $Cn$ could be replaced by $t_0$ from~\eqref{eq:tzero}.} 
  
  Consequently, for $\lambda = o(1/(n p_k))= o(n^{k-1})$, we have $E[T] \le (1 + o(1)) \frac{1}{p_k}$, and for $\lambda = \Omega(1/(n p_k)) = \Omega(n^{k-1})$, we have $E[T] = O(\lambda n)$.
\end{theorem}

We note that when $\lambda = o(n^{k-1})$, $\lambda \ge K \ln n$ with $K$ a sufficiently large constant, $\mu \le \frac{1}{(1+\delta)e} \lambda$, and $k \le 0.1n$, our upper bound and our lower bound of Theorem~\ref{thm:lower} agree including the leading constant. So we have a precise runtime analysis in this regime. 

We did not try to find the maximal range of parameters in which the runtime is $(1 \pm o(1)) \frac{1}{p_k}$. From~\cite{RoweS14} (see the discussion at the beginning of Section~\ref{sec:lower}) it is clear that for $\lambda \le c \ln n$, $c$ a sufficiently small constant, the runtime is $\exp(\Omega(n^C))$, where $C$ is a constant that depends on $c$. For $\mu \ge \frac{1}{(1-\delta)e} \lambda$, Lehre~\cite{Lehre10} gives an exponential lower bound. The restriction to $k \le 0.1n$ is most likely not necessary, but the range of larger $k$ appears not to be overly interesting given that the super-exponential lower bound from the case $0.1n$ still applies. 

When $\lambda = \Omega(n^{k-1})$, besides being a possibly unrealistically large population size, our time estimate of $O(n)$ iterations is the same as the best known upper bound for the runtime of the \mclea on the \onemax test function~\cite{CorusDEL18}. Since this analysis works with the natural partition into $\Theta(n)$ fitness levels, the runtime order of $O(n)$ shows that each fitness level is gained in an amortized constant number of iterations. This speed of progress on a difficult problem like jump functions again indicates that the offspring population size $\lambda$ here is chosen too large.

The exact order of magnitude of the runtime of the \mclea on \onemax is still an open problem. The upper bound proven for the \mplea in~\cite{AntipovDFH18}, which in our setting simplifies to $O(\frac{n \log n}{\lambda} + n \frac{\log\log \lambda/\mu}{\log \lambda/\mu})$, indicates that there could be some (but not much) room for improvement. So clearly, the next progress here should rather be for the \onemax function than for jump functions.

The result in Theorem~\ref{thm:upper} above improves over the $O(n^k + n\lambda + \lambda \log \lambda)$ upper bound for the runtime of the \mclea on $\jump_{nk}$ proven in~\cite{CorusDEL18} (see Theorem~\ref{thm:oldjump}) in three ways. First, as discussed above, we make the leading constant precise (and tight for large ranges of the parameters). Second, we obtain a better, namely at most linear, dependence of the runtime on $\lambda$. Third, we reduce the minimum offspring population size required for the result to hold, which is $\Omega(k \log n)$ in~\cite{CorusDEL18} and $\Omega(\log n)$ in our result.

\subsection{Level-based Analyses}\label{ssec:level}

A central step in our proof is an analysis of how the \mclea progresses to a parent population consisting only of individuals on the local optimum. Since the \mclea is a non-elitist algorithm, this asks for tools like the ones introduced by Lehre~\cite{Lehre11} and then improved by various authors~\cite{DangL16algo,CorusDEL18,DoerrK21}. Unfortunately, all these results are formulated for the problem of finding one individual of a certain minimum quality. Consequently, they all cannot be directly employed to analyze the time needed to have the full parent population consist of individuals of at least a certain quality. Fortunately, in their proofs all previous level-based analyses proceed by analyzing the time until a certain number of individuals of a certain quality have been obtained and then building on this with an analysis on how better individuals are generated. Among the previous works it appears that~\cite{DoerrK21} is the one that makes this argumentation most explicit, whereas the other works with their intricate potential function arguments give less insight into the working principles of the process.

For this reason, we build now on~\cite{DoerrK21}. To avoid restating an essentially unchanged proof from~\cite{DoerrK21}, we instead first state the level-based theorem shown in~\cite{DoerrK21}, explain where the different expressions in the runtime estimate stem from, and then state without explicit proof the level-based result we need. With the explanations given beforehand, we feel that the interested reader easily can see from~\cite{DoerrK21} why our level-based theorem is correct.

The general setup of \emph{level-based theorems for population processes} is as follows. There is a ground set $\X$, which will be search space $\{0,1\}^n$ in our applications. On this ground set, a population-based Markov process $(P_t)$ is defined. We consider populations of fixed size $\lambda$, which may contain elements several times (multi-sets). We write $\X^\lambda$ to denote the set of all such populations. We only consider Markov processes where each element of the next population is sampled independently (with repetition). That is, for each population $P \in \X^\lambda$, there is a distribution $D(P)$ on $\X$ such that given $P_t$, the next population $P_{t+1}$ consists of $\lambda$ elements of $\X$, each chosen independently from the distribution $D(P_t)$. We do not make any assumptions on the initial population $P_0$.

In the level-based setting, we assume that there is a partition of $\X$ into \emph{levels} $A_1, \dots, A_m$. Based on information in particular on how individuals in different levels are generated, we aim for an upper bound on the first time such that the population contains an element of the highest level $A_m$. Now the level-based theorem shown in~\cite{DoerrK21} is as follows.

\begin{theorem}[Level-based theorem]\label{thm:level}
Consider a population process as described above.

Let $(A_1,\ldots,A_m)$ be a partition of $\X$. Let $A_{\ge j} := \bigcup_{i=j}^m A_i$ for all $j \in [1..m]$. Let $z_1,\ldots,z_{m-1},\delta \in (0,1]$, and let $\gamma_0 \in (0,\frac{1}{1+\delta}]$  with $\gamma_0 \lambda \in \Z$. Let $D_0 = \min\{\lceil 100/\delta \rceil,\gamma_0 \lambda\}$ and $c_1 = 56\,000$. Let
$$
t_0 = \frac{7000}{\delta} \left(m + \frac{1}{1-\gamma_0} \sum_{j=1}^{m-1} \log^0_2\left(\frac{2\gamma_0\lambda}{1+\frac{z_j \lambda}{D_0}}\right) + \frac{1}{\lambda} \sum_{j=1}^{m-1}\frac{1}{z_j} \right),
$$
where $\log^0_2(x) := \max\{0,\log_2(x)\}$ for all $x \in \R$. Assume that for any population $P \in \X^\lambda$ the following three conditions are satisfied.
\begin{description}
	\item[(G1)] For each level $j \in [1..m-1]$, if $|P \cap A_{\geq j}| \geq \gamma_0 \lambda {/4}$, then
	$$
	\Pr_{y \sim D(P)} [y \in A_{\geq j+1}] \geq z_j.
	$$
	\item[(G2)] For each level $j \in [1..m-2]$ and all $\gamma \in (0,\gamma_0]$, if $|P \cap A_{\geq j}| \geq \gamma_0 \lambda {/4}$ and ${|P \cap A_{\geq j+1}| \geq \gamma \lambda}$, then
	$$
	\Pr_{y \sim D(P)} [y \in A_{\geq j+1}] \geq (1+\delta)\gamma.
	$$ 
	\item[(G3)] The population size $\lambda$ satisfies
	$$
	\lambda \geq {\frac{256}{\gamma_0 \delta} \ln \left(8 t_0 \right)}.
	$$
\end{description}

Then $T := \min\{\lambda t \mid P_t \cap A_m \neq \emptyset\}$ satisfies
\begin{align*}
E[T] 
 & \leq 8\lambda t_0 = c_1 \frac{\lambda}{\delta} \left(m + \frac{1}{1-\gamma_0} \sum_{j=1}^{m-2} \log^0_2\left(\frac{2\gamma_0\lambda}{1+\frac{z_j \lambda}{D_0}}\right) + \frac 1 \lambda \sum_{j=1}^{m-1}\frac{1}{z_j} \right).
\end{align*}
\end{theorem}

Let us explain where the time bound stated in this theorem stems from. We argue in terms of iterations now, not in terms of search point evaluations. Then the time bound is $8t_0$ with $t_0$ as defined in the theorem. The main argument of the proof given in~\cite{DoerrK21} is as follows. Let us, in the next three paragraphs, say that a population $P$ is \emph{well-established} on level $j$ if $|P \cap A_{\ge j}| \ge \gamma_0 \lambda / 4$. Now condition \Gzwei imposes that if the current population is well-established on level $j$, then the number of individuals on level $j+1$ or higher increases, in expectation, by a factor of $1+\delta$ until at least $\gamma_0 \lambda$ such individuals are in the population. It appears natural (and is true, but not trivial to prove) that it takes roughly $\log_{1+\delta}(\gamma_0 \lambda) \approx \frac 1\delta \log(\gamma_0 \lambda)$ iterations from the first individual on level $j+1$ to having at least $\gamma_0 \lambda$ individuals on this level. This explains roughly the middle term in the definition of $t_0$. Without going into details, we remark that the extra $\frac{z_j \lambda}{D_0}$ expression exploits that when generating individuals on a higher level (as described in \Geins) is easy, then we can assume that we do not start with a single individual on level $j+1$, but with roughly $z_j \lambda$ individuals. Consequently, we need the factor-$(1+\delta)$ growth only to go from $z_j \lambda$ to $\gamma_0 \lambda$ individuals. 

The remaining term $m + \frac 1 \lambda \sum_{j=1}^{m-1}\frac{1}{z_j}$ accounts for the time needed to generate the first individuals on a higher level. Given that the population is well-established on some level~$j$, by \Geins the probability that a new individual is on level $j+1$ or higher is at least~$z_j$. Since we generate $\lambda$ individuals in each step, the time to find an individual on a higher level (tacitly assuming that we stay well-established on level $j$, which is ensured by \Gdrei via Martingale concentration arguments) is at most $\lceil X/\lambda\rceil \le 1 + X/\lambda$, where $X$ is geometrically distributed with success probability $z_j$ and thus expectation $\frac 1 {z_j}$. 

This explanation of the definition of $t_0$ motivates that we can extend the result of Theorem~\ref{thm:level} to statements on how long it takes to have a certain level well-established (or even filled with at least $\gamma_0 \lambda$ individuals). This is what we do now. We omit the formal proof, but invite the reader to consult the proof in~\cite{DoerrK21}, which immediately yields our claim.

\begin{corollary}[Level-based theorem for filling sub-optimal levels]\label{cor:level}
Let a population process be given as described above.

Let $(A_1,\ldots,A_m)$ be a partition of $\X$. Let $A_{\ge j} := \bigcup_{i=j}^m A_i$ for all $j \in [1..m]$. Let $z_1,\ldots,z_{m-1},\delta \in (0,1]$, and let $\gamma_0 \in (0,\frac{1}{1+\delta}]$  with $\gamma_0 \lambda \in \Z$. Let $D_0 = \min\{\lceil 100/\delta \rceil,\gamma_0 \lambda\}$ and $c_1 = 56\,000$. 

Let $\ell \in [1..m-1]$ and 
$$
t_0(\ell) = \frac{7000}{\delta} \left(m + \frac{1}{1-\gamma_0} \sum_{j=1}^{\ell-1} \log^0_2\left(\frac{2\gamma_0\lambda}{1+\frac{z_j \lambda}{D_0}}\right) + \frac{1}{\lambda} \sum_{j=1}^{\ell-1}\frac{1}{z_j} \right),
$$
where $\log^0_2(x) := \max\{0,\log_2(x)\}$ for all $x \in \R$. Assume that for any population $P \in \X^\lambda$ the following three conditions are satisfied.
\begin{description}
	\item[(G1)] For each level $j \in [1..\ell-1]$, if $|P \cap A_{\geq j}| \geq \gamma_0 \lambda /4$, then
	$$
	\Pr_{y \sim D(P)} [y \in A_{\geq j+1}] \geq z_j.
	$$
	\item[(G2)] For each level $j \in [1..\ell-2]$ and all $\gamma \in (0,\gamma_0]$, if $|P \cap A_{\geq j}| \geq \gamma_0 \lambda {/4}$ and $|P \cap A_{\geq j+1}| \geq \gamma \lambda$, then
	$$
	\Pr_{y \sim D(P)} [y \in A_{\geq j+1}] \geq (1+\delta)\gamma.
	$$ 
	\item[(G3)] The population size $\lambda$ satisfies
	$$
	\lambda \geq {\frac{256}{\gamma_0 \delta} \ln \left(8 t_0(\ell) \right)}.
	$$
\end{description}

Then $T := \min\{\lambda t \mid |P_t \cap A_{\ge \ell}| \ge \gamma_0 \lambda\}$ satisfies
\begin{align*}
E[T] 
 & \leq 8\lambda t_0(\ell) = c_1 \frac{\lambda}{\delta} \left(m + \frac{1}{1-\gamma_0} \sum_{j=1}^{\ell-1} \log^0_2\left(\frac{2\gamma_0\lambda}{1+\frac{z_j \lambda}{D_0}}\right) + \frac 1 \lambda \sum_{j=1}^{\ell-1}\frac{1}{z_j} \right).
\end{align*}
\end{corollary}

\subsection{Proof of the Upper Bound}

We are now ready to prove our upper bound result. We start by giving a brief outline of the main arguments. We use our variant of the level-based theorem to argue that from any possible state of the algorithm, it takes an expected number of $O(n)$ iterations to reach a parent population that consists only of individuals in the local optimum (or the global optimum, but since we are done then, we can ignore this case). We call this an \emph{almost perfect} population. From this point on, we cannot use the level-based method anymore, since the small probability for going from the local to the global optimum would require a large value of $\lambda$, a requirement we try to avoid. This requirement is necessary in the level-based method because there one tries to ensure that once a decent number of individuals are on at least a certain level, this state is never lost. When $\lambda$ is only logarithmic in $n$, there is an inverse-polynomial probability to completely lose a level. Since for, say, $k = \Theta(n)$, we expect a runtime of roughly $n^k / \lambda$, in this time it will regularly happen that we lose a level, including the cases that we lose a level in each of several iterations or that we lose several levels at once.

We overcome this difficulty with a restart argument. Since the probability for such an undesirable event is only inverse-polynomial in $n$, we see that we keep an almost perfect population for at least $n^2$ iterations (with high probability). Since it took us only $O(n)$ iterations to reach (or regain) an almost perfect population, we obtain that in all but a lower order fraction of the iterations we have an almost perfect parent population. Hence apart from this lower order performance loss, we can assume that we are always in an almost perfect population. From such a state, we reach the optimum in one iteration with probability $1 - (1 - p_k)^\lambda$, which quickly leads to the claimed result.

We now state the formal proof, which makes this proof sketch precise and adds a few arguments not discussed so far.

\begin{proof}[Proof of Theorem~\ref{thm:upper}]
Since $n^n$ is a trivial upper bound for the expected runtime of any evolutionary algorithm creating all individuals as random search points or via standard bit mutation with mutation rate $\frac 1n$, simply because each of these search points with probability at least $n^{-n}$ is the optimum\footnote{This argument, ignoring however the initial search points, was made already in~\cite{DrosteJW02} to show this runtime bound for the \oea.}, we can assume that $k < n$.

Let $m = n+1$ and let $A_1, \dots, A_m$ be the partition of $\{0,1\}^n$ into the fitness levels of $\jump_{nk}$, that is, for all $i \in [1..m-1]$ we have $A_i = \{x \in \{0,1\}^n \mid f(x) = i\}$ and for $i = m$ we have $A_i = \{(1, \dots, 1)\}$. In particular, for all $i \in [1..m-1]$ and all $x \in A_i$, $y \in A_{i+1}$ we have $f(x) < f(y)$. Also, $A_m$ consists of the unique optimum and $A_{m-1}$ consists of all local optima.
  
  Consider a run of the \mclea on $\jump_{nk}$. As in Algorithm~\ref{alg:algo}, we denote by $P_t$ the population (of size $\mu$) selected in iteration~$t$, which serves as parent population in generation $t+1$. Let $P_0$ denote the initial population. We denote by $Q_t$ the offspring population (of size $\lambda$) generated in iteration $t$. Hence $P_t$ consists of $\mu$ best individuals chosen from $Q_t$. For the sake of a smooth presentation, let $Q_0$ be a population obtained from $P_0$ by adding $\lambda - \mu$ random search points of minimal fitness. Note that we can again assume that $P_0$ is obtained from $Q_0$ by selecting $\mu$ best individuals. 
  
  We say that a parent population $P_t$ is \emph{almost perfect} if $P_t \subseteq A_{\ge m-1}$. Note that this is equivalent to having $|Q_t \cap A_{\ge m-1}| \ge  \mu$.
  
  \textbf{Step 1:} We first argue that for any time $s \ge 0$ and regardless of what is~$Q_s$, the first time $S \ge s$ such that $P_{S}$ is almost perfect satisfies $E[S-s] \le 8 t_0$, where
  \[
t_0 = \frac{10^4}{\delta} \left(m + \frac{1}{1-\gamma_0} \sum_{j=1}^{m-2} \log^0_2\left(\frac{2\gamma_0\lambda}{1+\frac{z_j \lambda}{D_0}}\right) + \frac{1}{\lambda} \sum_{j=1}^{m-2}\frac{1}{z_j} \right).
\]

To ease the notation, we assume that $s=0$. To estimate $S$, we apply our variant of the level-based theorem (Corollary~\ref{cor:level}) to the process $(Q_t)_{t \ge 0}$. Since optimizing jump functions up to the local optimum is very similar to optimizing the \onemax function, this analysis is similar to an analogous analysis for \onemax (where we note that the work~\cite{CorusDEL18} proving the previous-best result for \onemax for most details of the proof refers to the not very detailed conference paper~\cite{Lehre11}).
  
  We choose suitable parameters to use the level-theorem. For ${j \in [1..k-1]}$, this corresponds to the fitness levels lying in the gap region of $\jump_{nk}$, let $z_j = \frac{1}{4} \frac {n-j}{en}$. For $j \in [k..m-2]$, here $A_j$ consists of the search points $x$ with $\OM(x) = j-k$, we let $z_j = \frac{1}{4} \frac{n - (j-k)}{en}$. Note that for $j \in [1..k-1]$, we have $z_j \ge \frac{1}{4} \frac {k+1-j}{en}$, and hence 
  \begin{equation}
  \sum_{j=1}^{m-2} \frac 1 {z_j} \le 4en \sum_{i=2}^{n} \frac 1 i \le 4en \ln n,\label{eq:harmonic}
  \end{equation}
  recalling that the harmonic number $H_n = \sum_{i=1}^n \frac 1i$ satisfies $H_n \le \ln(n)+ 1$, see, e.g., \cite[(1.4.12)]{Doerr20bookchapter}. Note also, for later, that for any $j$ we have $z_j \ge \frac 14 \frac{n-j}{en}$.
  
  Let 
  $\gamma_0$ be such that $\gamma_0 \lambda = \lfloor \frac{\lambda}{(1+\delta)e} \rfloor$. Note that by our assumption that $\lambda$ is large, $\gamma_0 \lambda$ is an integer greater than one as required in Corollary~\ref{cor:level}. Also, $\gamma_0 \le \frac 1 {1+\delta}$ as required. By our assumption $\lambda \ge (1+\delta) e \mu$, we have $\gamma_0 \lambda \ge \mu$. Trivially, $\gamma_0 \le \frac{1}{(1+\delta)e} \le \frac 1e$. Let $D_0 = \min\{\lceil 100 / \delta \rceil, \gamma_0 \lambda\}$. 
  
  We check that the conditions \textbf{(G1)} to \textbf{(G3)} of Corollary~\ref{cor:level} are satisfied for $\ell = m-1$. To show \Geins and \Gzwei,  let $t \ge 0$ be any iteration. 
  
  \textbf{(G1):} Let $j \in [1..m-2]$ such that $|Q_t \cap A_{\ge j}| \ge \gamma_0 \lambda / 4$. We need to show that an offspring $y$ generated in iteration $t+1$ is in $A_{\ge j+1}$ with probability at least $z_j$. Let first $j \ge k$, that is, $A_j$ is not a level in the gap. Let $y$ be an offspring generated in iteration $t+1$ and let $x \in P_t$ be its random parent, which we can assume to be not the optimum as otherwise we would be done already. Since $\gamma_0 \lambda / 4 \ge \mu/4$, there are at least $\mu/4$ individuals in $P_t \cap A_{\ge j}$. Hence with probability at least $1/4$, we have $x \in A_{\ge j}$. In this case, we have 
  \begin{align*}
  \Pr[y \in A_{j+1}] 
  &\ge \min\left\{\frac 1e, \left(1-\frac 1n\right)^{n-1} \frac{n - (j-k)}{n}\right\}\\ 
  &\ge \frac{n - (j-k)}{en},
  \end{align*} 
  where the first case refers to $x$ already being in $A_{\ge j+1}$, that is, $1  \le j+1-k \le \OM(x) \le n-1$, and uses Lemma~\ref{lem:gleich}, and where the second case refers to $x \in A_j$, that is, $\OM(x) = j-k$. In total, we have $\Pr[y \in A_{\ge j+1}] \ge \frac{1}{4} \frac{n - (j-k)}{en} = z_j$. If $j < k$, we proceed analogously with the only exception that, since in the first case we could have $\OM(x) = 0$, we now estimate $\Pr[\OM(y) = \OM(x)] \ge (1-\frac 1n)^n$. We thus obtain $\Pr[y \in A_{j+1} \mid x \in A_{\ge j}] \ge \min\{(1-\frac 1n)^n, (1-\frac 1n)^{n-1} \frac{n-j}{n}\} \ge (1-\frac 1n)^{n-1} \frac{n-j}{n} \ge \frac {n-j}{en}$. Consequently, now $\Pr[y \in A_{\ge j+1}] \ge \frac{1}{4} \frac{n-j}{en} = z_j$. 
  
  \textbf{(G2):} Let $j \in [1..m-2]$ such that $|Q_t \cap A_{\ge j}| \ge \gamma_0 \lambda / 4$. Let $\gamma \in (0,\gamma_0]$ such that $|Q_t \cap A_{\ge j+1}| \ge \gamma \lambda$. We need to show that an offspring $y$ generated in iteration $t+1$ is in $A_{\ge j+1}$ with probability at least $(1+\delta) \gamma$. Let $x$ be a parent selected uniformly at random from $P_t$, where again we assume that $P_t$ contains no optimal solution. There are at least $\min\{\gamma \lambda, \mu\} \ge \min\{\gamma (1+\delta) e \mu, \mu\} = \gamma (1+\delta) e \mu$ individuals in $P_t \cap A_{\ge j+1}$. Hence with probability at least $\gamma (1+\delta) e $, we have $x \in A_{\ge j+1}$. In this case, $\Pr[y \in A_{\ge j+1}] \ge \Pr[\OM(y) = \OM(x)] \ge \frac 1e$ by Lemma~\ref{lem:gleich} when $\OM(x) \neq 0$. When $\OM(x) = 0$, then $\Pr[y \in A_{\ge j+1}] \ge \Pr[\OM(y) \in \{0,1\}] \ge \Pr[\forall i \in [2..n] : x_i = y_i] = (1-\frac 1n)^{n-1} \ge \frac 1e$, where the first estimate uses our assumption $k<n$. Hence without conditioning on $x \in A_{\ge j+1}$, we have $\Pr[y \in A_{\ge j+1}] \ge \gamma (1+\delta)e \cdot \frac 1e \ge (1+\delta) \gamma$. 
  
  \textbf{(G3):} We first estimate $t_0$. We recall that $\gamma_0 \le \frac 1e$ and $z_j \ge \frac 14 \frac{n-j}{en}$ for all $j$. Thus, for all $j \in [1..m-2]$, we have 
  \[\log^0_2\left(\frac{2\gamma_0\lambda}{1+\frac{z_j \lambda}{D_0}}\right) \le \log^0_2\left(\frac{(2/e) D_0}{z_j}\right) \le \log_2\left(\frac{8 D_0 n}{n-j}\right). 
  \]
  Consequently, 
  \begin{align*}
  \frac{1}{1-\gamma_0} & \sum_{j=1}^{m-2} \log^0_2\left(\frac{2\gamma_0\lambda}{1+\frac{z_j \lambda}{D_0}}\right) 
   \le \frac{e}{e-1} \sum_{j=1}^{m-2} \log_2\left(\frac{8 D_0 n}{n-j}\right) \\
  & = \frac{e}{e-1} \log_2 \left(\prod_{j=1}^{m-2} \frac{8 D_0 n}{n-j} \right) \le \frac{e}{e-1} \log_2 \left(\frac{(8 D_0n)^n}{n!}\right) \\
  &\le \frac{e}{e-1} \log_2 \left(\frac{(8 D_0n)^n}{(n/e)^n}\right) = \frac{e}{e-1} n \log_2 (8eD_0),
  \end{align*} 
  where we used the well-known estimate $n! \ge (n/e)^n$, see, e.g.,~\cite[(1.4.13)]{Doerr20bookchapter}.
  
  From this and~\eqref{eq:harmonic}, we obtain 
  \begin{align}
  t_0 &= \frac{10^4}{\delta} \left(m + \frac{1}{1-\gamma_0} \sum_{j=1}^{m-2} \log^0_2\left(\frac{2\gamma_0\lambda}{1+\frac{z_j \lambda}{D_0}}\right) + \frac{1}{\lambda} \sum_{j=1}^{m-2}\frac{1}{z_j} \right)\nonumber\\
  & \le \frac{10^4}{\delta} \left( m + \frac{e}{e-1} n \log_2 (8eD_0) + \frac 1\lambda 4en \ln n \right) = O\left(n\right),\label{eq:tzero}
  \end{align}
  where the asymptotic estimate uses the fact that $\lambda = \Omega(\log n)$. This shows~\Gdrei. 
  
  From \Geins to \Gdrei, Corollary~\ref{cor:level} shows that after an expected number of $8t_0$ iterations, we have reached an offspring population $Q_t$ with $|Q_t \cap A_{\ge m-1}| \ge \gamma_0 \lambda = \mu$ and thus an almost perfect population $P_t$.

\textbf{Step 2:}
We now show that when $P_t$ contains only local optima, then with probability at least $1 - n^{-2}$, the same is true for $P_{t+1}$ or the global optimum has been found. Indeed, by our initial assumption $k < n$, the search points on the local optimum have a \onemax-value between $1$ and $n-1$. Hence Lemma~\ref{lem:gleich} implies that $X := |Q_{t+1} \cap A_{\ge m-1}|$ follows a binomial law with parameters $\lambda$ and success probability at least $\frac 1e$. By the additive Chernoff bound (Theorem~\ref{tprobchernoffadditive01}), we have
\begin{align*}
\Pr\left[X \le \frac{\lambda}{e (1+\delta)}\right] &= \Pr\left[X \le E[X] - \frac{\delta}{1+\delta} E[X]\right]\\
& \le \exp\left( - \frac{2(\frac{\delta}{1+\delta} E[X])^2}{\lambda}\right) \\
& = \exp\left(-\frac 2 {e^2} \left(\frac{\delta}{1+\delta}\right)^2 \lambda\right) \le n^{-2}
\end{align*}
by our assumption that $\lambda \ge K \ln(n)$ with a constant $K$ sufficiently large. Since $\mu \le \frac{\lambda}{e (1+\delta)}$, we have $|P_{t+1} \cap A_{\ge m-1}| < \mu$ only if $X < \mu \le \frac{\lambda}{e (1+\delta)}$. As just computed, this happens with probability at most $n^{-2}$.

\textbf{Step 3:}
Recall that $p_k = (1-\frac 1n)^{n-k} n^{-k}$ is the probability to generate the optimum from a parent on the local optimum. Let $T_0 = \min\left\{\left\lceil \sqrt{t_0 / \lambda p_k} \, \right\rceil, \left\lfloor n^{3/2} \right\rfloor \right\}$. We call a \emph{phase} of a run of the algorithm an interval of (i)~first all iterations until we have an almost perfect parent population $P_t$, and then (ii)~another exactly $T_0$ iterations. We assume here for simplicity that we continue to run the algorithm even when it found the optimum; in such a case, we replace such an optimum immediately with a random search point on the local optimum. Since we are interested in the first time an optimum is found, this modification does not change our results. By definition and step~1 above, the expected length of a phase is at most $8 t_0 + T_0$ regardless of how this phase starts.

We call a phase \emph{regular} if after reaching an almost perfect parent population $P_t$ we never (in the following exactly $T_0$ iterations) encounter a parent population $P_t$ that is not almost perfect. By a simple union bound and step~2 above, each phase is regular with probability at least $1 - n^{-2} T_0 \ge 1 - n^{-1/2}$, regardless how the phase started and how it reached an almost perfect parent population. 

A regular phase is \emph{successful} if it finds the optimum at least once. Since in a regular phase at least $\lambda T_0$ times an offspring is generated from a parent on the local optimum (which results in the global optimum with probability $p_k$), and since these offspring are generated independently, the probability for a regular phase to be not successful is at most $(1 - p_k)^{\lambda T_0}$, which is at most $\frac{1}{1 + p_k \lambda T_0}$ by an elementary estimate stated as Lemma~8 in~\cite{RoweS14}. Since thus a regular phase is successful with probability at least $1 - \frac{1}{1 + p_k \lambda T_0} = \frac{p_k \lambda T_0}{1 + p_k \lambda T_0}$, it takes an expected number of $\frac{1 + p_k \lambda T_0}{p_k \lambda T_0} = 1 + \frac{1}{p_k \lambda T_0}$ regular phases to find the optimum. Since phases are regular with probability at least $1- n^{-1/2}$, it takes an expected number of at most $\frac{1}{1- n^{-1/2}} \cdot (1 + \frac{1}{p_k \lambda T_0})$ phases to find the optimum. By Wald's equation, these take an expected number of at most $\frac{1}{1- n^{-1/2}} \cdot (1 + \frac{1}{p_k \lambda T_0}) \cdot (8 t_0 + T_0)$ iterations. We estimate
\begin{align*}
\left(1 + \frac{1}{p_k \lambda T_0}\right) &\cdot (8 t_0 + T_0)  = 8t_0 + T_0 + \frac{8t_0}{p_k \lambda T_0} + \frac{1}{p_k \lambda}\\
& \le 8 t_0 + \sqrt{\frac{t_0}{p_k \lambda}} + 1 + 8 \sqrt{\frac{t_0}{p_k \lambda}} + \frac{8t_0}{p_k \lambda \lfloor n^{3/2} \rfloor} + \frac{1}{p_k \lambda}.
\end{align*}
Recalling that $t_0 = O(n)$, we note that this expression is $O(n)$ when $p_k \lambda = \Omega(1/n)$ and $(1+o(1)) \frac{1}{p_k \lambda}$ when $p_k \lambda = o(1/n)$. Recalling further that each iteration contains $\lambda$ fitness evaluations, see also~\eqref{eq:runtime}, the claim follows. 
\end{proof}

\section{Conclusion}

In this work, we observed that for all reasonable parameter values, the \mclea cannot optimize jump functions faster than the \mplea. The \mclea thus fails to profit from its ability to leave local optima to inferior solutions. While we prove this absence of advantage formally only for the basic \mclea and jump functions (which constitute, however, a standard algorithm and a classic benchmark), we feel that our proofs do not suggest that this result is caused by very special characteristics of the \mclea or the jump functions, but that it rather follows from the fact that leaving a local optimum having moderate radius of attraction via comma selection is generally difficult because, relatively independent of the population sizes, there is a strong drift towards the local optimum. We do not show such a strong drift when $\lambda < 2 \mu$, but in this case the selection pressure is known to be so low that no efficient optimization is possible.

Overall, this work suggests that the role of comma selection in evolutionary computation deserves some clarification. Interesting directions for future research could be to try to find convincing examples where comma selection is helpful or a general result going beyond particular examples that shows in which situations comma selection cannot speed up the optimization of multimodal objective functions. From a broader perspective, any result giving a mildly general advice which of the existing approaches to cope with local optima are preferable in which situations, would be highly desirable. The new analysis methods developed in this work, which can yield precise runtime bounds for non-elitist population processes and negative drift situations, could be helpful as they now allow to prove or disprove constant-factor advantages.

\subsection*{Acknowledgment}

This work was supported by a public grant as part of the
Investissements d'avenir project, reference ANR-11-LABX-0056-LMH,
LabEx LMH.

}

\newcommand{\etalchar}[1]{$^{#1}$}


\begin{thebibliography}{FGQW18b}

\bibitem[ABD20]{AntipovBD20gecco}
Denis Antipov, Maxim Buzdalov, and Benjamin Doerr.
\newblock Fast mutation in crossover-based algorithms.
\newblock In {\em Genetic and Evolutionary Computation Conference, GECCO 2020},
  pages 1268--1276. {ACM}, 2020.

\bibitem[AD11]{AugerD11}
Anne Auger and Benjamin Doerr, editors.
\newblock {\em Theory of Randomized Search Heuristics}.
\newblock World Scientific Publishing, 2011.

\bibitem[AD18]{AntipovD18}
Denis Antipov and Benjamin Doerr.
\newblock Precise runtime analysis for plateaus.
\newblock In {\em Parallel Problem Solving From Nature, PPSN 2018, Part~{II}},
  pages 117--128. Springer, 2018.

\bibitem[ADFH18]{AntipovDFH18}
Denis Antipov, Benjamin Doerr, Jiefeng Fang, and Tangi Hetet.
\newblock Runtime analysis for the ${(\mu+\lambda)}$ {EA} optimizing
  {O}ne{M}ax.
\newblock In {\em Genetic and Evolutionary Computation Conference, GECCO 2018},
  pages 1459--1466. ACM, 2018.

\bibitem[ADK20]{AntipovDK20}
Denis Antipov, Benjamin Doerr, and Vitalii Karavaev.
\newblock The $(1 + (\lambda,\lambda))$ {GA} is even faster on multimodal
  problems.
\newblock In {\em Genetic and Evolutionary Computation Conference, GECCO 2020},
  pages 1259--1267. {ACM}, 2020.

\bibitem[ADY19]{AntipovDY19}
Denis Antipov, Benjamin Doerr, and Quentin Yang.
\newblock The efficiency threshold for the offspring population size of the
  ${(\mu,\lambda)}$ {EA}.
\newblock In {\em Genetic and Evolutionary Computation Conference, {GECCO}
  2019}, pages 1461--1469. {ACM}, 2019.

\bibitem[AL14]{AlanaziL14}
Fawaz Alanazi and Per~Kristian Lehre.
\newblock Runtime analysis of selection hyper-heuristics with classical
  learning mechanisms.
\newblock In {\em Congress on Evolutionary Computation, {CEC} 2104}, pages
  2515--2523. IEEE, 2014.

\bibitem[BDN10]{BottcherDN10}
S\"untje B{\"o}ttcher, Benjamin Doerr, and Frank Neumann.
\newblock Optimal fixed and adaptive mutation rates for the {L}eading{O}nes
  problem.
\newblock In {\em Parallel Problem Solving from Nature, PPSN 2010}, pages
  1--10. Springer, 2010.

\bibitem[CDEL18]{CorusDEL18}
Dogan Corus, Duc{-}Cuong Dang, Anton~V. Eremeev, and Per~Kristian Lehre.
\newblock Level-based analysis of genetic algorithms and other search
  processes.
\newblock {\em {IEEE} Transactions on Evolutionary Computation}, 22:707--719,
  2018.

\bibitem[COY17]{CorusOY17}
Dogan Corus, Pietro~S. Oliveto, and Donya Yazdani.
\newblock On the runtime analysis of the {O}pt-{IA} artificial immune system.
\newblock In {\em Genetic and Evolutionary Computation Conference, {GECCO}
  2017}, pages 83--90. {ACM}, 2017.

\bibitem[COY18]{CorusOY18fast}
Dogan Corus, Pietro~S. Oliveto, and Donya Yazdani.
\newblock Fast artificial immune systems.
\newblock In {\em Parallel Problem Solving from Nature, {PPSN} 2018, Part
  {II}}, pages 67--78. Springer, 2018.

\bibitem[DFK{\etalchar{+}}16]{DangFKKLOSS16}
Duc{-}Cuong Dang, Tobias Friedrich, Timo K{\"{o}}tzing, Martin~S. Krejca,
  Per~Kristian Lehre, Pietro~S. Oliveto, Dirk Sudholt, and Andrew~M. Sutton.
\newblock Escaping local optima with diversity mechanisms and crossover.
\newblock In {\em Genetic and Evolutionary Computation Conference, GECCO 2016},
  pages 645--652. {ACM}, 2016.

\bibitem[DFK{\etalchar{+}}18]{DangFKKLOSS18}
Duc{-}Cuong Dang, Tobias Friedrich, Timo K{\"{o}}tzing, Martin~S. Krejca,
  Per~Kristian Lehre, Pietro~S. Oliveto, Dirk Sudholt, and Andrew~M. Sutton.
\newblock Escaping local optima using crossover with emergent diversity.
\newblock {\em {IEEE} Transactions on Evolutionary Computation}, 22:484--497,
  2018.

\bibitem[DG13]{DoerrG13algo}
Benjamin Doerr and Leslie~A. Goldberg.
\newblock Adaptive drift analysis.
\newblock {\em Algorithmica}, 65:224--250, 2013.

\bibitem[DJW02]{DrosteJW02}
Stefan Droste, Thomas Jansen, and Ingo Wegener.
\newblock On the analysis of the (1+1) evolutionary algorithm.
\newblock {\em Theoretical Computer Science}, 276:51--81, 2002.

\bibitem[DK15]{DoerrK15}
Benjamin Doerr and Marvin K{\"{u}}nnemann.
\newblock Optimizing linear functions with the $(1+\lambda)$ evolutionary
  algorithm---different asymptotic runtimes for different instances.
\newblock {\em Theoretical Computer Science}, 561:3--23, 2015.

\bibitem[DK21]{DoerrK21}
Benjamin Doerr and Timo K{\"{o}}tzing.
\newblock Multiplicative up-drift.
\newblock {\em Algorithmica}, 2021.
\newblock \href {https://doi.org/10.1007/s00453-020-00775-7}
  {\path{doi:10.1007/s00453-020-00775-7}}.

\bibitem[DL16a]{DangL16algo}
Duc{-}Cuong Dang and Per~Kristian Lehre.
\newblock Runtime analysis of non-elitist populations: from classical
  optimisation to partial information.
\newblock {\em Algorithmica}, 75:428--461, 2016.

\bibitem[DL16b]{DangL16ppsn}
Duc-Cuong Dang and Per~Kristian Lehre.
\newblock Self-adaptation of mutation rates in non-elitist populations.
\newblock In {\em Parallel Problem Solving from Nature, PPSN 2016}, pages
  803--813. Springer, 2016.

\bibitem[DLMN17]{DoerrLMN17}
Benjamin Doerr, Huu~Phuoc Le, R\'egis Makhmara, and Ta~Duy Nguyen.
\newblock Fast genetic algorithms.
\newblock In {\em Genetic and Evolutionary Computation Conference, GECCO 2017},
  pages 777--784. {ACM}, 2017.

\bibitem[DLN19]{DangLN19}
Duc{-}Cuong Dang, Per~Kristian Lehre, and Phan Trung~Hai Nguyen.
\newblock Level-based analysis of the univariate marginal distribution
  algorithm.
\newblock {\em Algorithmica}, 81:668--702, 2019.

\bibitem[DLOW18]{DoerrLOW18}
Benjamin Doerr, Andrei Lissovoi, Pietro~S. Oliveto, and John~Alasdair
  Warwicker.
\newblock On the runtime analysis of selection hyper-heuristics with adaptive
  learning periods.
\newblock In {\em Genetic and Evolutionary Computation Conference, GECCO 2018},
  pages 1015--1022. ACM, 2018.

\bibitem[DN20]{DoerrN20}
Benjamin Doerr and Frank Neumann, editors.
\newblock {\em Theory of Evolutionary Computation---Recent Developments in
  Discrete Optimization}.
\newblock Springer, 2020.
\newblock Also available at
  \url{https://cs.adelaide.edu.au/~frank/papers/TheoryBook2019-selfarchived.pdf}.

\bibitem[Doe19a]{Doerr19tcs}
Benjamin Doerr.
\newblock Analyzing randomized search heuristics via stochastic domination.
\newblock {\em Theoretical Computer Science}, 773:115--137, 2019.

\bibitem[Doe19b]{Doerr19foga}
Benjamin Doerr.
\newblock An exponential lower bound for the runtime of the compact genetic
  algorithm on jump functions.
\newblock In {\em Foundations of Genetic Algorithms, FOGA 2019}, pages 25--33.
  {ACM}, 2019.

\bibitem[Doe19c]{Doerr19gecco}
Benjamin Doerr.
\newblock A tight runtime analysis for the {cGA} on jump functions: {EDA}s can
  cross fitness valleys at no extra cost.
\newblock In {\em Genetic and Evolutionary Computation Conference, GECCO 2019},
  pages 1488--1496. {ACM}, 2019.

\bibitem[Doe20a]{Doerr20gecco}
Benjamin Doerr.
\newblock Does comma selection help to cope with local optima?
\newblock In {\em Genetic and Evolutionary Computation Conference, GECCO 2020},
  pages 1304--1313. {ACM}, 2020.

\bibitem[Doe20b]{Doerr20ppsnLB}
Benjamin Doerr.
\newblock Lower bounds for non-elitist evolutionary algorithms via negative
  multiplicative drift.
\newblock In {\em Parallel Problem Solving From Nature, PPSN 2020, Part~II},
  pages 604--618. Springer, 2020.

\bibitem[Doe20c]{Doerr20bookchapter}
Benjamin Doerr.
\newblock Probabilistic tools for the analysis of randomized optimization
  heuristics.
\newblock In Benjamin Doerr and Frank Neumann, editors, {\em Theory of
  Evolutionary Computation: Recent Developments in Discrete Optimization},
  pages 1--87. Springer, 2020.
\newblock Also available at \url{https://arxiv.org/abs/1801.06733}.

\bibitem[DWY21]{DoerrWY21}
Benjamin Doerr, Carsten Witt, and Jing Yang.
\newblock Runtime analysis for self-adaptive mutation rates.
\newblock {\em Algorithmica}, 83:1012--1053, 2021.

\bibitem[Ere99]{Eremeev99}
Anton~V. Eremeev.
\newblock Modeling and analysis of genetic algorithm with tournament selection.
\newblock In {\em Artificial Evolution, AE 1999}, pages 84--95. Springer, 1999.

\bibitem[FGQW18a]{FriedrichGQW18heavysubm}
Tobias Friedrich, Andreas G{\"{o}}bel, Francesco Quinzan, and Markus Wagner.
\newblock Evolutionary algorithms and submodular functions: Benefits of
  heavy-tailed mutations.
\newblock {\em CoRR}, abs/1805.10902, 2018.

\bibitem[FGQW18b]{FriedrichGQW18}
Tobias Friedrich, Andreas G{\"{o}}bel, Francesco Quinzan, and Markus Wagner.
\newblock Heavy-tailed mutation operators in single-objective combinatorial
  optimization.
\newblock In {\em Parallel Problem Solving from Nature, PPSN 2018, Part {I}},
  pages 134--145. Springer, 2018.

\bibitem[FKK{\etalchar{+}}16]{FriedrichKKNNS16}
Tobias Friedrich, Timo K{\"{o}}tzing, Martin~S. Krejca, Samadhi Nallaperuma,
  Frank Neumann, and Martin Schirneck.
\newblock Fast building block assembly by majority vote crossover.
\newblock In {\em Genetic and Evolutionary Computation Conference, GECCO 2016},
  pages 661--668. {ACM}, 2016.

\bibitem[FQW18]{FriedrichQW18}
Tobias Friedrich, Francesco Quinzan, and Markus Wagner.
\newblock Escaping large deceptive basins of attraction with heavy-tailed
  mutation operators.
\newblock In {\em Genetic and Evolutionary Computation Conference, {GECCO}
  2018}, pages 293--300. {ACM}, 2018.

\bibitem[GKS99]{GarnierKS99}
Josselin Garnier, Leila Kallel, and Marc Schoenauer.
\newblock Rigorous hitting times for binary mutations.
\newblock {\em Evolutionary Computation}, 7:173--203, 1999.

\bibitem[GW17]{GiessenW17}
Christian Gie{\ss}en and Carsten Witt.
\newblock The interplay of population size and mutation probability in the ${(1
  + \lambda)}$ {EA} on {OneMax}.
\newblock {\em Algorithmica}, 78:587--609, 2017.

\bibitem[Haj82]{Hajek82}
Bruce Hajek.
\newblock Hitting-time and occupation-time bounds implied by drift analysis
  with applications.
\newblock {\em Advances in Applied Probability}, 13:502--525, 1982.

\bibitem[HJKN08]{HappJKN08}
Edda Happ, Daniel Johannsen, Christian Klein, and Frank Neumann.
\newblock Rigorous analyses of fitness-proportional selection for optimizing
  linear functions.
\newblock In {\em Genetic and Evolutionary Computation Conference, {GECCO}
  2008}, pages 953--960. ACM, 2008.

\bibitem[Hoe63]{Hoeffding63}
Wassily Hoeffding.
\newblock Probability inequalities for sums of bounded random variables.
\newblock {\em Journal of~the American Statistical Association}, 58:13--30,
  1963.

\bibitem[HS18]{HasenohrlS18}
V{\'{a}}clav Hasen{\"{o}}hrl and Andrew~M. Sutton.
\newblock On the runtime dynamics of the compact genetic algorithm on jump
  functions.
\newblock In {\em Genetic and Evolutionary Computation Conference, {GECCO}
  2018}, pages 967--974. {ACM}, 2018.

\bibitem[HY01]{HeY01}
Jun He and Xin Yao.
\newblock Drift analysis and average time complexity of evolutionary
  algorithms.
\newblock {\em Artificial Intelligence}, 127:51--81, 2001.

\bibitem[Jan05]{Jansen05}
Thomas Jansen.
\newblock A comparison of simulated annealing with a simple evolutionary
  algorithm.
\newblock In {\em Foundations of Genetic Algorithms, {FOGA} 2005}, pages
  37--57. Springer, 2005.

\bibitem[Jan13]{Jansen13}
Thomas Jansen.
\newblock {\em Analyzing Evolutionary Algorithms -- The Computer Science
  Perspective}.
\newblock Springer, 2013.

\bibitem[JJW05]{JansenJW05}
Thomas Jansen, Kenneth A.~De Jong, and Ingo Wegener.
\newblock On the choice of the offspring population size in evolutionary
  algorithms.
\newblock {\em Evolutionary Computation}, 13:413--440, 2005.

\bibitem[JS07]{JagerskupperS07}
Jens J{\"a}gersk{\"u}pper and Tobias Storch.
\newblock When the plus strategy outperforms the comma strategy and when not.
\newblock In {\em Foundations of Computational Intelligence, FOCI 2007}, pages
  25--32. IEEE, 2007.

\bibitem[JW02]{JansenW02}
Thomas Jansen and Ingo Wegener.
\newblock The analysis of evolutionary algorithms -- a proof that crossover
  really can help.
\newblock {\em Algorithmica}, 34:47--66, 2002.

\bibitem[JW07]{JansenW07}
Thomas Jansen and Ingo Wegener.
\newblock A comparison of simulated annealing with a simple evolutionary
  algorithm on pseudo-{B}oolean functions of unitation.
\newblock {\em Theoretical Computer Science}, 386:73--93, 2007.

\bibitem[K{\"{o}}t16]{Kotzing16}
Timo K{\"{o}}tzing.
\newblock Concentration of first hitting times under additive drift.
\newblock {\em Algorithmica}, 75:490--506, 2016.

\bibitem[Leh10]{Lehre10}
Per~Kristian Lehre.
\newblock Negative drift in populations.
\newblock In {\em Parallel Problem Solving from Nature, PPSN 2010}, pages
  244--253. Springer, 2010.

\bibitem[Leh11]{Lehre11}
Per~Kristian Lehre.
\newblock Fitness-levels for non-elitist populations.
\newblock In {\em Genetic and Evolutionary Computation Conference, {GECCO}
  2011}, pages 2075--2082. {ACM}, 2011.

\bibitem[Len20]{Lengler20bookchapter}
Johannes Lengler.
\newblock Drift analysis.
\newblock In Benjamin Doerr and Frank Neumann, editors, {\em Theory of
  Evolutionary Computation: Recent Developments in Discrete Optimization},
  pages 89--131. Springer, 2020.
\newblock Also available at \url{https://arxiv.org/abs/1712.00964}.

\bibitem[LOW17]{LissovoiOW17}
Andrei Lissovoi, Pietro~S. Oliveto, and John~Alasdair Warwicker.
\newblock On the runtime analysis of generalised selection hyper-heuristics for
  pseudo-{B}oolean optimisation.
\newblock In {\em Genetic and Evolutionary Computation Conference, {GECCO}
  2017}, pages 849--856. {ACM}, 2017.

\bibitem[LOW19]{LissovoiOW19}
Andrei Lissovoi, Pietro~S. Oliveto, and John~Alasdair Warwicker.
\newblock On the time complexity of algorithm selection hyper-heuristics for
  multimodal optimisation.
\newblock In {\em Conference on Artificial Intelligence, {AAAI} 2019}, pages
  2322--2329. {AAAI} Press, 2019.

\bibitem[LS18]{LenglerS18}
Johannes Lengler and Angelika Steger.
\newblock Drift analysis and evolutionary algorithms revisited.
\newblock {\em Combinatorics, Probability {\&} Computing}, 27:643--666, 2018.

\bibitem[NOW09]{NeumannOW09}
Frank Neumann, Pietro~S. Oliveto, and Carsten Witt.
\newblock Theoretical analysis of fitness-proportional selection: landscapes
  and efficiency.
\newblock In {\em Genetic and Evolutionary Computation Conference, {GECCO}
  2009}, pages 835--842. {ACM}, 2009.

\bibitem[NS20]{NguyenS20}
Phan Trung~Hai Nguyen and Dirk Sudholt.
\newblock Memetic algorithms outperform evolutionary algorithms in multimodal
  optimisation.
\newblock {\em Artificial Intelligence}, 287:103345, 2020.

\bibitem[NW10]{NeumannW10}
Frank Neumann and Carsten Witt.
\newblock {\em Bioinspired Computation in Combinatorial Optimization --
  Algorithms and Their Computational Complexity}.
\newblock Springer, 2010.

\bibitem[OPH{\etalchar{+}}18]{OlivetoPHST18}
Pietro~S. Oliveto, Tiago Paix{\~{a}}o, Jorge~P{\'{e}}rez Heredia, Dirk Sudholt,
  and Barbora Trubenov{\'{a}}.
\newblock How to escape local optima in black box optimisation: when
  non-elitism outperforms elitism.
\newblock {\em Algorithmica}, 80:1604--1633, 2018.

\bibitem[OW11]{OlivetoW11}
Pietro~S. Oliveto and Carsten Witt.
\newblock Simplified drift analysis for proving lower bounds in evolutionary
  computation.
\newblock {\em Algorithmica}, 59:369--386, 2011.

\bibitem[OW12]{OlivetoW12}
Pietro~S. Oliveto and Carsten Witt.
\newblock Erratum: Simplified drift analysis for proving lower bounds in
  evolutionary computation.
\newblock {\em CoRR}, abs/1211.7184, 2012.

\bibitem[OW15]{OlivetoW15}
Pietro~S. Oliveto and Carsten Witt.
\newblock Improved time complexity analysis of the simple genetic algorithm.
\newblock {\em Theoretical Computer Science}, 605:21--41, 2015.

\bibitem[PHST17]{PaixaoHST17}
Tiago Paix{\~{a}}o, Jorge~P{\'{e}}rez Heredia, Dirk Sudholt, and Barbora
  Trubenov{\'{a}}.
\newblock Towards a runtime comparison of natural and artificial evolution.
\newblock {\em Algorithmica}, 78:681--713, 2017.

\bibitem[RA19]{RoweA19}
Jonathan~E. Rowe and Aishwaryaprajna.
\newblock The benefits and limitations of voting mechanisms in evolutionary
  optimisation.
\newblock In {\em Foundations of Genetic Algorithms, {FOGA} 2019}, pages
  34--42. {ACM}, 2019.

\bibitem[RS14]{RoweS14}
Jonathan~E. Rowe and Dirk Sudholt.
\newblock The choice of the offspring population size in the (1, {\(\lambda\)})
  evolutionary algorithm.
\newblock {\em Theoretical Computer Science}, 545:20--38, 2014.

\bibitem[Sud09]{Sudholt09}
Dirk Sudholt.
\newblock The impact of parametrization in memetic evolutionary algorithms.
\newblock {\em Theoretical Computer Science}, 410:2511--2528, 2009.

\bibitem[Sud13]{Sudholt13}
Dirk Sudholt.
\newblock A new method for lower bounds on the running time of evolutionary
  algorithms.
\newblock {\em {IEEE} Transactions on Evolutionary Computation}, 17:418--435,
  2013.

\bibitem[Weg05]{Wegener05}
Ingo Wegener.
\newblock Simulated annealing beats {M}etropolis in combinatorial optimization.
\newblock In {\em Automata, Languages and Programming, {ICALP} 2005}, pages
  589--601. Springer, 2005.

\bibitem[Wit06]{Witt06}
Carsten Witt.
\newblock Runtime analysis of the ($\mu$ + 1) {EA} on simple pseudo-{B}oolean
  functions.
\newblock {\em Evolutionary Computation}, 14:65--86, 2006.

\bibitem[Wit13]{Witt13}
Carsten Witt.
\newblock Tight bounds on the optimization time of a randomized search
  heuristic on linear functions.
\newblock {\em Combinatorics, Probability {\&} Computing}, 22:294--318, 2013.

\bibitem[Wit19]{Witt19}
Carsten Witt.
\newblock Upper bounds on the running time of the univariate marginal
  distribution algorithm on {OneMax}.
\newblock {\em Algorithmica}, 81:632--667, 2019.

\bibitem[WQT18]{WuQT18}
Mengxi Wu, Chao Qian, and Ke~Tang.
\newblock Dynamic mutation based {P}areto optimization for subset selection.
\newblock In {\em Intelligent Computing Methodologies, {ICIC} 2018, Part
  {III}}, pages 25--35. Springer, 2018.

\bibitem[WVHM18]{WhitleyVHM18}
Darrell Whitley, Swetha Varadarajan, Rachel Hirsch, and Anirban Mukhopadhyay.
\newblock Exploration and exploitation without mutation: solving the jump
  function in ${\Theta(n)}$ time.
\newblock In {\em Parallel Problem Solving from Nature, {PPSN} 2018, Part
  {II}}, pages 55--66. Springer, 2018.

\bibitem[WZD21]{WangZD21}
Shouda Wang, Weijie Zheng, and Benjamin Doerr.
\newblock Choosing the right algorithm with hints from complexity theory.
\newblock In {\em International Joint Conference on Artificial Intelligence,
  {IJCAI} 2021}, pages 1697--1703. ijcai.org, 2021.

\end{thebibliography}

\end{document}